\documentclass[sn-mathphys-num]{sn-jnl}

\usepackage{graphicx}%
\usepackage{multirow}%
\usepackage{amsmath,amssymb,amsfonts}%
\usepackage{amsthm}%
\usepackage{mathrsfs}%
\usepackage[title]{appendix}%
\usepackage{xcolor}%
\usepackage{textcomp}%
\usepackage{manyfoot}%
\usepackage{booktabs}%
\usepackage{algorithm}%
\usepackage{algorithmicx}%
\usepackage{algpseudocode}%
\usepackage{listings}%
\usepackage{setspace}
\usepackage{subcaption}

\newtheorem{theorem}{Theorem}
\newtheorem{lemma}{Lemma}
\newtheorem{definition}{Definition}%

\newtheorem{assumption}{Assumption}[section]

\usepackage{etoolbox}
\apptocmd{\thebibliography}{\setlength{\itemsep}{0pt plus 0.5pt}}{}{}

\begin{document}

\title[Article Title]{On the Convergence Analysis of Over-Parameterized Variational Autoencoders: A Neural Tangent Kernel Perspective}


\author[1]{\fnm{Li} \sur{Wang}}\email{li.wang1@csiro.au}

\author*[2]{\fnm{Wei} \sur{Huang}}\email{wei.huang.vr@riken.jp}

\affil[1]{\orgname{CSIRO Space and Astronomy}, \orgaddress{\street{26 Dick Perry Ave}, \city{Kensington}, \postcode{6151}, \state{WA}, \country{Australia}}}

\affil*[2]{\orgname{RIKEN AIP}, \city{Tokyo},  \country{Japan}}


\abstract{
Variational Auto-Encoders (VAEs) have emerged as powerful probabilistic models for generative tasks. However, their convergence properties have not been rigorously proven. The challenge of proving convergence is inherently difficult due to the highly non-convex nature of the training objective and the implementation of a Stochastic Neural Network (SNN) within VAE architectures. This paper addresses these challenges by characterizing the optimization trajectory of SNNs utilized in VAEs through the lens of Neural Tangent Kernel (NTK) techniques. These techniques govern the optimization and generalization behaviors of ultra-wide neural networks. We provide a mathematical proof of VAE convergence under mild assumptions, thus advancing the theoretical understanding of VAE optimization dynamics. Furthermore, we establish a novel connection between the optimization problem faced by over-parameterized SNNs and the Kernel Ridge Regression (KRR) problem. Our findings not only contribute to the theoretical foundation of VAEs but also open new avenues for investigating the optimization of generative models using advanced kernel methods. Our theoretical claims are verified by experimental simulations.
}

\keywords{Variational Auto-encoder, Stochastic Neural Network, Neural Tangent Kernel}

\maketitle
\section{Introduction}

Variational Autoencoders (VAEs) \cite{kingma2013auto} have garnered significant interest and have been applied across a diverse array of applications, ranging from image generation and style transfer \cite{radford2015unsupervised,ref:vae-image, wang2022pruning} to natural language processing \cite{bowman2015generating}. VAEs aim to learn a compressed yet structured latent representation of input data by maximizing the Evidence Lower BOund (ELBO), thereby facilitating the reconstruction of the original data. Unlike traditional autoencoders \cite{ng2011sparse,tschannen2018recent}, VAEs focus on learning the distribution of latent codes, enabling the generation of new samples from this distribution.
The dimensionality of the latent space is dictated by data complexity, model objectives, and task-specific needs, ranging from a few to several thousand dimensions. Larger latent spaces can encode more information and provide better disentanglement learning~\cite{song2019latent,lim2020deep, zhang2022weighted}, a finding that our experiments also support (see Figures \ref{fig:measure} and \ref{fig:add}). Concurrently, there is an intuitive belief that a larger latent space may pose challenges to training, such as issues with non-convergence or slow convergence rates.

On the other hand, despite the widespread application of VAEs, our theoretical understanding of the training dynamics remains limited. Investigating the optimization of Deep VAEs theoretically is notoriously challenging, as training deep neural networks involves non-convex optimization of a high-dimensional objective function.  The complexity of this optimization problem is further exacerbated by the incorporation of stochastic neural networks (SNNs) in VAEs, which introduces additional stochasticity into the training process. Several studies have attempted to shed light on this problem from different perspectives. For instance, He et al. \cite{he2019lagging} conducted an empirical investigation of the learning dynamics of deep
VAEs to study the posterior collapse. Lucas et al. \cite{lucas2019don} presented a simple and intuitive analysis of linear VAEs to explain the same collapse. Moreover, Koehler et al. \cite{koehler2021variational} analyzed the training dynamics, offering insights into implicit bias convergence for linear VAEs. However, much of the existing research either leans heavily on empirical simulations or centers around linear VAEs, leaving the broader success of VAEs insufficiently explained.

To address concerns about the convergence in high-dimensional latent spaces in VAEs, in this work, we introduce a novel convergence analysis for VAE training dynamics, specifically when an over-parameterized stochastic neural network serves as its model. While the convergence properties of deterministic neural networks have been extensively explored \cite{jacot2018neural,allen2019convergence,du2018gradient,du2019gradient,huang2020neural,huang2021towards,zou2020gradient,chen2021equivalence,chen2019much}, the convergence behavior of SNNs in VAE remains less understood. Our approach leverages non-asymptotic analysis of dynamical systems, allowing us to examine the behavior of over-parameterized VAEs during training. We demonstrate that the convergence outcome aligns with solving a kernel ridge regression under certain mild assumptions. To our knowledge, this is the first rigorous analysis of the convergence behavior of over-parameterized VAEs. We further validate our theoretical insights through experiments on various image generation tasks. In summary, our key contributions are as follows:

\begin{itemize}
    \item We establish a non-asymptotic convergence analysis for over-parameterized SNNs. Specifically, we investigate the convergence rate of the optimization algorithm used to train the VAE. 
    
    \item We link the optimization of over-parameterized SNNs with kernel ridge regression, shedding light on the regularization effects of the KL penalty in VAEs.

    \item Theoretically, we prove that VAEs with high-dimensional latent spaces can converge, providing a theoretical foundation for employing large latent spaces in VAEs to capture more information.

\end{itemize}

\section{Related Work}

\paragraph{Convergence Analysis of Over-parameterized Neural Networks}
The convergence analysis of over-parameterized neural networks (NNs) has become an important topic in deep learning research. In a seminal paper, Jacot et al. \cite{jacot2018neural} showed that the optimization behavior of infinitely-wide NNs can be described using a kernel function called neural tangent Kerenl (NTK). This kernel simplifies the optimization dynamics into a linear system that is more tractable. The NTK provides a way to explicitly characterize the dynamics of the neural network during training and to analyze its convergence behavior \cite{lee2019wide,yang2019scaling}. Additionally, a series of studies \cite{du2018gradient,du2019gradient,arora2019exact,arora2019fine,allen2019convergence,zou2020gradient} have presented convergence results of over-parameterized networks through a non-asymptotic lens. Furthermore, the Rademacher complexity analysis characterized the generalization ability of trained over-parameterized NNs on unseen data \cite{cao2019generalization,arora2019fine}. In addition, NTK has been widely applied to different deep network structures, aiding in understanding their optimization dynamics. This includes convolutional networks \cite{arora2019exact}, orthogonally initialized NN \cite{huang2020neural}, graph neural networks \cite{du2019graph}, active learning \cite{wang2022deep}, transformer \cite{hron2020infinite}, neural architecture search \cite{chen2022deep}, and GAN \cite{franceschi2022neural}.

Among existing studies of training dynamics of over-parameterized networks, the works of \cite{nguyen2021benefits,ziyin2022stochastic,huang2023analyzing,clerico2023wide} are the most aligned with our research. Nguyen et al. \cite{nguyen2021benefits}  explored the gradient dynamics of over-parameterized auto-encoders (AE) and provided a rigorous proof for the linear convergence of gradient descent in the context of AEs. However, their techniques cannot be directly applied to variational auto-encoders (VAEs) because of the additional randomness introduced by stochastic neural networks. In a separate study, Liu et al. \cite{ziyin2022stochastic} examined the predictive variance of stochastic neural networks. They demonstrated that as the width of an optimized stochastic neural network approaches infinity, its predictive variance on the training set diminishes to zero. While their work sheds light on the behavior of stochastic neural networks in the infinite-width limit, they have not shown the convergence of infinitely-wide neural networks, which is one of the most desirable perspectives of studying a NN. Two other notable studies  \cite{huang2023analyzing,clerico2023wide} approached SNNs within the PAC-Bayes framework, leveraging the NTK. However, the SNN structure in our VAE research differs from the PAC-Bayes framework, particularly in how stochasticity is introduced in the latent layer.


\paragraph{Theoretical study of VAEs}

While VAEs have been successfully applied in various domains, their theoretical properties are still not fully understood. Several recent works have attempted to provide a theoretical understanding of VAEs. For instance, recent works by \cite{alemi2018fixing,dai2019diagnosing,rolinek2019variational} refereed to information theory, deriving variational bounds on the mutual information between the input and the latent variable and the objective function. One work by Lucas et al. \cite{lucas2019don} provided an intuitive explanation for the posterior collapse phenomenon in VAEs. They analyze linear VAEs and show that the posterior collapse can be attributed to the low-rank structure of the encoder. In addition, Kumar et al. \cite{kumar2020implicit} presented an approximation of VAE objective function consisting of deterministic auto-encoding objective plus analytic regularizers that depend on the Hessian or Jacobian of the decoding model. Nakagawa et al. \cite{nakagawa2021quantitative} provided a quantitative understanding of the VAE property through the differential-geometric and information-theoretic interpretations of VAE. Moreover, \cite{wipf2023marginalization,dai2021value,dai2020usual} are not around the optimization dynamics but they study problems of optimization landscape.
In contrast, our work studies the training dynamics of over-parameterized VAEs with the non-linear activation, emphasizing the challenges on the non-linear activation and the complicated optimization behavior.

\section{Problem Setup and Preliminary}

\subsection{Notation}

In this work, we adopt a standard notation to represent vectors, matrices, and scalars. Specifically, we use bold-faced letters for vectors and matrices and non-bold letters for scalars.
To denote the Euclidean norm of a vector or the spectral norm of a matrix, we use the notation $\| \cdot \|_2$. The Frobenius norm of a matrix is represented by $\| \cdot \|_F$. We use the notation $[n] = {1,2,\ldots,n }$ to represent the set of integers from 1 to $n$. Besides, we represent a matrix as a set of row vectors, i.e., 
${\bf W} = [{\bf w}^\top_1, {\bf w}^\top_2, \dots, {\bf w}^\top_m ]^\top$, 
where ${\bf w}_r$ with $r \in [m]$ is a column vector of the matrix. Finally, we denote the least eigenvalue of a matrix by $\lambda_0(\boldsymbol{\Theta})$, which is equivalent to $\lambda_{\min}( \boldsymbol{\Theta} )$. 


\subsection{Variational Auto-encoder}

A Variational auto-encoder (VAE) \cite{kingma2013auto}, as a directed probabilistic graphical model (DPGM), is designed to learn a latent variable model. Its primary objective is to maximize the log-likelihood of the training data $\{{\bf x}_i \}^n_{i=1}$ via variational inference, where $n$ is the number of training samples. The VAE introduces a distribution $q_{\phi}(\bf z|x)$ to approximate the intractable true posterior $p(\bf z|x)$, where $\phi$ are neural network parameters that can be learned in the encoder. 
Then, the decoder takes ${\bf z}$ as input to generate ${\bf x}^\prime$ as a reconstruction for ${\bf x}$.

The common training objective of the VAE is to maximize the Evidence Lower Bound (ELBO), given by:
{\small \begin{equation} \label{eq:target}
L_{elbo} = \frac{1}{n} \sum_{i=1}^n \mathbb{E}_\mathbf{z}[\log  p_{\theta}({\bf x}'_i|{\bf z})] - \mathrm{KL}(q_{\phi}({\bf z}| {\bf x}_i)\|p({\bf z})),
\end{equation} }
where ${\bf z} \sim q_{\phi}({\bf z}|{\bf x}_i)$, and $\phi$ and $\theta$
represent the parameters in encoder and decoder, respectively. The first term in the ELBO measures the reconstruction loss between the generated ${\bf x}^\prime$ and the original ${\bf x}$. The second term represents the Kullback-Leibler (KL) divergence between the approximate posterior $q(\bf z|x)$ and the prior $p({\bf z})$, where $p({\bf z})$ is often chosen to be an isotropic multivariate Gaussian distribution.

\subsection{Stochastic Neural Network and Objective Function}

\begin{figure}
	\centering  
	\includegraphics[width=0.75\linewidth]{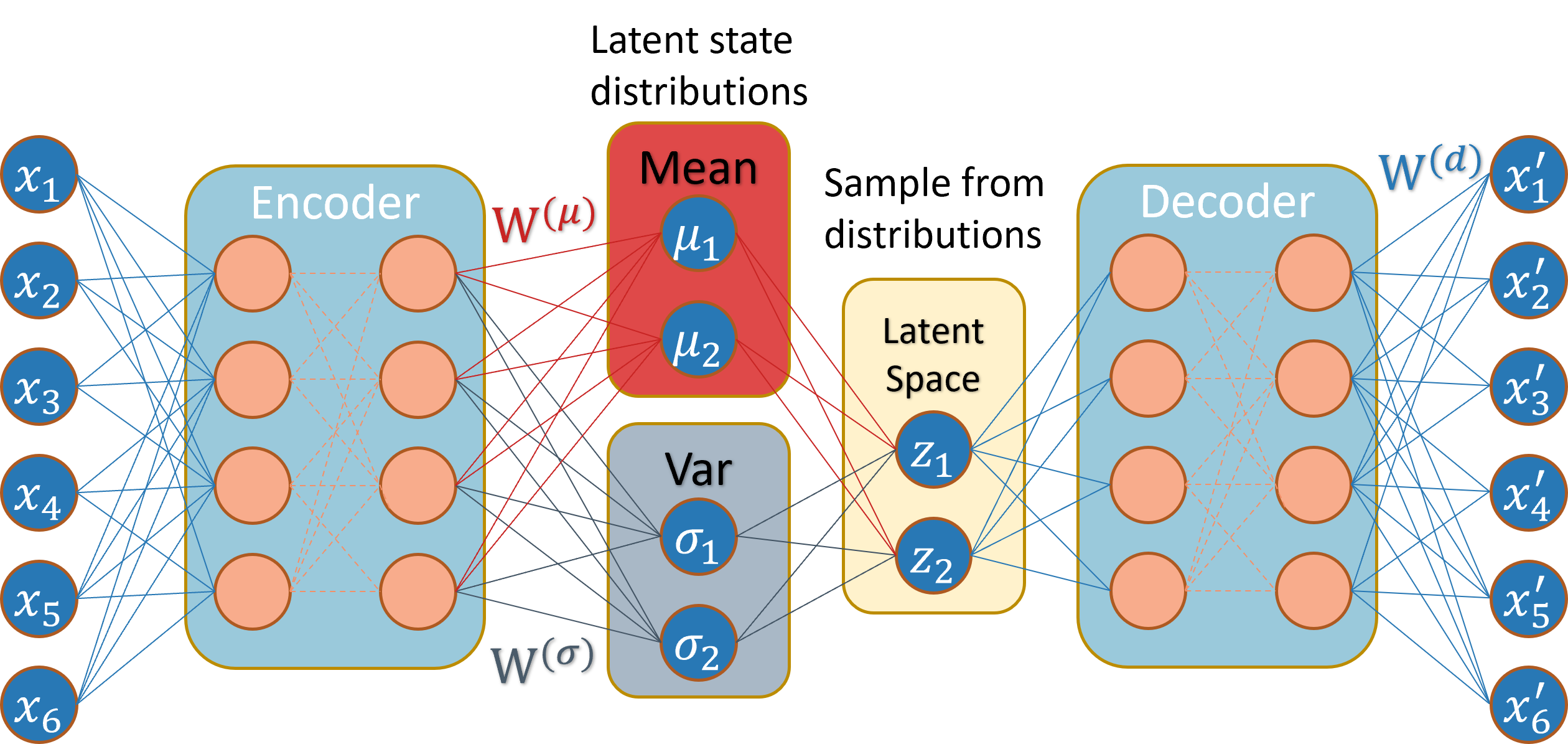}
	\caption{Architecture of Variational Auto-Encoder.}
	\label{VAE}
\end{figure}

Consider a stochastic neural network (SNN) $\mathbf{f} \in \mathbb{R}^d$, where $d$ is the input dimension. In the context of this work, our SNN is defined as follows:
{\small \begin{equation} \label{eq:net}
 {\small \begin{aligned}
    \mathbf{f}({\bf x}) = \frac{1}{\sqrt{m}} ({\bf W}^{(d)})^\top \psi (\sigma({\bf z})), \quad
     {\bf z} \sim \mathcal{N}({\bf W}^{(\mu)} {\bf x}^{(e)},  \mathrm{diag}({\bf W}^{(\sigma)} {\bf x}^{(e)})),
     \end{aligned} }
\end{equation} }
where ${\bf x}^{(e)} \in \mathbb{R}^d$ is the encoded representation derived from the input $\mathbf{x}$, ${\bf W}^{(\mu)}, {\bf W}^{(\sigma)} \in \mathbb{R}^{m \times d}$ are weight matrices employed to construct the latent Gaussian representation. Here  $m$ represents the width of the network, indicating the number of neurons, $\sigma(\cdot)$ is the non-linear activation function, $\psi(\cdot) $ is the decoder representation function, and ${\bf W}^{(d)} \in \mathbb{R}^{m \times d}$ is the linear weight matrix utilized in the final layer. A visual representation of the SNN under study is depicted in Figure \ref{VAE}.

In the construction of the latent representation, we employ the re-parametrization trick, a technique that allows for the backpropagation of gradients through random nodes. In particular, the latent variable can be expressed as:
{\small \begin{equation}\label{eq:repara}
{\small \begin{aligned}
{\bf z} = {\bf W}^{(\mu)} {\bf x}^{(e)} + ({\bf W}^{(\sigma)}  \odot  \boldsymbol{\zeta} ) {\bf x}^{(e)}, \quad \boldsymbol{\zeta} \sim \mathcal{N}( {\bf 0}, {\bf I} ), 
\end{aligned} }
\end{equation} }
where $\mathbf{W}^{(\mu)}$ and $\mathbf{W}^{(\sigma)}$ represent the mean and variance weights, respectively. Besides, $\boldsymbol{\zeta}$ is a random variable drawn from a standard normal distribution.

Given the structure of the SNN, our objective function considered in this work is defined as:
{\small \begin{equation} \label{eq:objective}
L = \frac{1}{n}\sum_{i=1}^n \big[  \ell(\hat{\mathbf{f}}( \mathbf{x}_i ), \mathbf{x}_i)  + \beta {\rm KL} \big( P(\mathbf{z}_i(t)) \| P(\mathbf{z}_i(0))  \big) \big],
\end{equation} }
where $\hat{\mathbf{f}}(\mathbf{x}_i) \triangleq \mathbb{E}_{\boldsymbol{\zeta}}[ \mathbf{f}( \mathbf{x}_i, \boldsymbol{\zeta}) ]$, and $\mathbf{z}_i(t)$ is the latent representation for input $\mathbf{x}_i$ at time $t$.
Besides, $\beta$ is an adjustable hyperparameter that balances latent channel capacity and independence constraints
with reconstruction accuracy \cite{higgins2016beta}. The first term $\sum_{i=1}^n \ell(\hat{\mathbf{f}}(\mathbf{x}_i ), \mathbf{x}_i)$ is called the reconstruction loss. In this study, we utilize the mean squared error as our reconstruction loss, following seminal theoretical works \cite{huang2023analyzing,du2019gradient,cao2019generalization,arora2019fine}. The second term ${\rm KL}(\cdot)$ is a Kullback–Leibler (KL) divergence, where prior distribution is the Gaussian distribution of latent variable at initialization, and the posterior is the distribution of latent variable after training, $ \mathbf{z}_i  \sim \mathcal{N}({\bf W}^{(\mu)} {\bf x}^{(e)}_i, \mathrm{diag}({\bf W}^{(\sigma)} {\bf x}^{(e)}_i))$. It's worth noting that our KL is tailored to align with our theoretical analysis for constructing kernel ridge regression.

To optimize the objective function given by (\ref{eq:objective}), we adopt a gradient descent rule:

{\small \begin{equation}
    {\small \begin{aligned}
     \mathbf{W}^{(s)}(t+1) & =   \mathbf{W}^{(s)}(t) - \eta \frac{ \partial  {L}(t)}{ \partial \mathbf{W}^{(s)}(t)}, \text{where } s\in \{\mu,\sigma,d\},\\
    \end{aligned} }
\end{equation} }
where $\eta$ is the learning rate. Note that while the weights in the encoder and decoder remain fixed, we specifically optimize the mean weights $\mathbf{W}^{(\mu)}$, variance weights $\mathbf{W}^{(\sigma)}$, and the weights in the final layer $\mathbf{W}^{(d)}$. This optimization strategy is primarily adopted for the sake of theoretical simplicity. It's worth noting that this choice does not compromise or alter our final conclusions.


\section{Theoretical Results}

In this section, we present our primary theoretical findings related to the optimization of the VAE's objective function. We start from the essential definitions and assumptions, later the convergence will be established. Finally, we prove the kernel ridge regression result through over-parameterization.

\subsection{Definition and Assumptions}

For the purpose of our optimization analysis, we introduce the concept of the neural tangent kernel for a stochastic neural network:
\begin{definition}[Stochastic Neural Tangent Kernel] \label{eq:ntk} 
The tangent kernels associated with output function at weights are defined as, 
{\small \begin{equation}\label{NTK_kernel}
{\small \begin{aligned}
\boldsymbol{\Theta}_{ik,jk'}^{(s)}  &=\nabla_{ {\bf W^{(s)} } }  \hat{f}_k ({\bf x}_i;t)^{\top} \nabla_{\bf W^{(s)}} \hat{f}_{k'}({\bf x}_j;t) \in \mathbb{R}, \text{where } s\in \{\mu,\sigma,d\}\\
\end{aligned} }
\end{equation} }
and $i,j \in [1,n]$ denote the index of input samples while $k, k' \in [1,d]$ represent the index of output functions. Furthermore, the NTK for the entire network is defined as $ \boldsymbol{\Theta} =\boldsymbol{\Theta}^{(\mu)}  + \boldsymbol{\Theta}^{(\sigma)} + \boldsymbol{\Theta}^{(d)}$.   
\end{definition}

A few remarks on Definition~\ref{eq:ntk} 
are in order. Unlike standard (deterministic) neural networks, the VAE comprises two sets of parameters in the latent layer, namely, ${\bf W }^{(\mu)}$ and ${\bf W }^{(\sigma)}$. Due to the reparameterization trick, gradient descent is executed on each of these parameters independently. Consequently, we observe two distinct tangent kernels corresponding to each parameter set. Secondly, The scenario with multiple outputs in variational autoencoder networks presents added complexity compared to networks with a single output  \cite{du2018gradient,arora2019fine}. Given that the output dimension of the stochastic neural network is $d$, the neural tangent kernel is a matrix of size $\mathbb{R}^{nd \times nd}$.
As we delve deeper in the subsequent sections, it will become evident that the non-diagonal NTK across the output index is zero, and the diagonal NTK remains consistent across the output index. This uniformity allows us to employ Kronecker products, facilitating the derivation of NTKs.

Next, we impose some technical conditions on the activation function, which is stated as follows:
\begin{assumption}  [Continuous and Partial Derivative Continuous]  \label{ass:activation} The activation function
$\sigma(x)$
 and its partial derivative $\frac{\partial \sigma(x)}{\partial x} $
 are continuous in $x$.
\end{assumption}
This assumption ensures that we can interchange the operations of integration and differentiation over the activation function. Subsequently, we present technical conditions on both the activation function and the decoder representation function:
\begin{assumption}  [$L$-Lipschitz and $\beta$-Smooth]  \label{ass:activation_lip} There exist  constants $\beta$ and $L$ such that for any $x,x' \in \mathbb{R}$:
{\small \begin{align*}
    \left| \sigma(x) - \sigma(x') \right| \le L \left| x -x' \right|, 
     \left| \sigma'(x) - \sigma'(x') \right| \le \beta \left| x -x' \right|, \\
      \left| \psi(x) - \psi(x') \right| \le L \left| x -x' \right|, 
     \left| \psi'(x) - \psi'(x') \right| \le \beta \left| x -x' \right|.
\end{align*} }
\end{assumption}
These conditions are important in demonstrating the stability of the training process within the framework of the NTK.

\subsection{Optimization analysis}

For the sake of simplification, we focus on the optimization of the stochastic neural network as described in (\ref{eq:net}), emphasizing solely on the reconstruction loss. This means we are setting aside the KL divergence term for the time being. Additionally, given that we're adopting a squared loss without KL divergence, the objective function (\ref{eq:objective}) reduces to:
{\small \begin{equation}
   L_{mse} = \frac{1}{2n}\sum_{i=1}^n   \left \| \hat{\mathbf{f}} (\mathbf{x}_i)- \mathbf{x}_i \right \|^2_2.  
\end{equation} }
Then the gradient flow dynamics of output function $\hat{f}_k$ are governed by:
{\small \begin{equation} \label{eq:dynamics}
{\small \begin{aligned}
 \frac{d  \hat{f}_k({\bf x}_i;t)}{d t} &  = \frac{1}{n}
   \sum_{j=1}^n \sum_{k'=1}^d \left({x}_{j,k'}-  \hat{f}_{k'}({\bf x}_j;t)\right) \boldsymbol{\Theta}_{ik,jk'}(t).
\end{aligned} }
\end{equation} }
Equation (\ref{eq:dynamics}) implies that the dynamics of output function are governed by the neural tangent kernels. Furthermore, as we will show later, the neural tangent kernels will stay constant during the training process in the infinite-width limit. In this way, Equation (\ref{eq:dynamics}) reduces to an ordinary differential equation (ODE):
{\small \begin{equation} \label{eq:limit_dynamics}
{\small \begin{aligned}
 \frac{d  \hat{f}_k({\bf x}_i;t)}{d t} &  = \frac{1}{n}
   \sum_{j=1}^n \sum_{k'=1}^d \left({x}_{j,k'}-  \hat{f}_{k'}({\bf x}_j;t) \right)  \boldsymbol{\Theta}^{(\infty)}_{ik,jk'},
\end{aligned} }
\end{equation} }
where we define the neural tangent kernel of an infinitely-wide SNN by:
{\small \begin{equation}
    \boldsymbol{\Theta}^{(\infty)} \triangleq \lim_{m \rightarrow \infty} \boldsymbol{\Theta} = \lim_{m \rightarrow \infty} \left(\boldsymbol{\Theta}^{(\mu)}+ \boldsymbol{\Theta}^{(\sigma)}+ \boldsymbol{\Theta}^{(d)}\right).
\end{equation} }
To demonstrate the convergence result induced by Equation (\ref{eq:limit_dynamics}), we perform  an in-depth concentration analysis. This analysis focuses on the convergence of stochastic neural networks in a non-asymptotic manner, i.e., with a large but finite width. We present our main result in the following theorem: 
\begin{theorem}\label{thm:opt} Assume the lowest eigenvalue of the limiting NTK is greater than zero, i.e., $\lambda_{0}(\boldsymbol{\Theta}^\infty) $ and $\| \mathbf{x}^{(e)}_i \|_2 = 1$ for $i\in [n]$.
Suppose the network's width $m = \Omega  \left( \max \left\{  \frac{n^5 d^3  }{\lambda_0^4 \delta^2 }  ,  \frac{n^2d^2}{\lambda_0}\log\frac{nd}{\delta}\right \} \right)$, then with probability at least $ 1- \delta $ over the random initialization we have,
{\small \begin{equation}
  L_{mse}(t)  \le \exp\left(-(\lambda_0/n) t \right)  L_{mse}(0).
\end{equation} }
\end{theorem}
The proof sketch of Theorem \ref{thm:opt} will be given in Section \ref{sec:proof}. 
Theorem \ref{thm:opt} establishes that if $m$ is large enough, the expected training error converges to zero at a linear rate. In particular, the least eigenvalue of NTK governs the convergence rate. 

\subsection{Regularization effect of KL divergence} \label{sec:kl}

By Theorem \ref{thm:opt}, we establish the global convergence of stochastic neural networks with a large width in VAE. Building on this foundation, we further consider full objective function (\ref{eq:objective}) which incorporates an additional KL divergence term.

After a detailed calculation of the KL divergence for two Gaussian distributions, we simplify our analysis by making certain assumptions. Specifically, we assume that $\mathbf{W}^{(\sigma)}$ remains constant and select a prior $\mathbf{x}^{(e)}_i$ such that the objective function (\ref{eq:objective}) is transformed to:
{\small \begin{equation} \label{eq:true_obj}
     {L}(t) =   \frac{1}{2n}   \left \|\hat{\mathbf{f}} (\mathbf{X};t)- \mathbf{X} \right \|^2_F +  \frac{\beta}{2} \left \| \mathbf{W}^{(\mu)}(t) - \mathbf{W}^{(\mu)}(0)  \right \|^2_F.
\end{equation} }
Building on this, we further analyze the regularization effect of the KL term when training VAEs and present our findings in the subsequent theorem:
\begin{theorem} \label{thm:solution}
 Suppose $m \ge {\rm poly}({n},1/\lambda_0, 1/\delta,1/ \mathcal{E})$ and the objective function follows the form (\ref{eq:true_obj}). When we only optimize the mean weight $\mathbf{W}^{(\mu)}$, for any test input $\mathbf{x}_{te} \in \mathbb{R}^d$ with probability at least $(1-\delta)$ over the random initialization, we have
{\small \begin{equation}
{\small \begin{aligned}
 \hat{\mathbf{f}}({\bf x}_{te},\infty) & = \boldsymbol{\Theta}^{(\mu)}({\bf x}_{te},{\bf X})(\boldsymbol{\Theta}^{(\mu)}({\bf X},{\bf X})+\beta {\bf I})^{-1} {\bf X}   \pm \mathcal{E}.
 \end{aligned} }
\end{equation} }
where $\mathcal{E}$ is the residual error term \textcolor{black}{and is upper bounded by $\mathcal{E}_{init} + \mathcal{E}_{\Theta} \frac{\sqrt{n}}{\lambda_0+  \beta  }$ with $  \|\hat{\mathbf{f}} \left(\boldsymbol{\theta}(0),\mathbf{x}_{te}\right)  \|_2 \le \mathcal{E}_{\rm init}$ and $\|\boldsymbol{\Theta}^\infty -\boldsymbol{\Theta}(t)  \|_2 \le \mathcal{E}_{\Theta}$}.
\end{theorem}
The proof of Theorem \ref{thm:solution} will be given in the Appendix. \textcolor{black}{Note that the error term is bounded by the difference between the output function of the finite network and the infinitely-wide network. This difference is further decomposed into the initial difference and the difference during training. The latter can be bounded by $\frac{\sqrt{n}}{\lambda_0 + \beta} \mathcal{E}_{\Theta}$, where $\sqrt{n}$ comes from the input and $\frac{1}{\lambda_0 + \beta}$ results from the integration over the training time.} Besides, \textcolor{black}{the necessity of fixing the variance weight in Theorem \ref{thm:solution} arises because we are seeking a closed-form solution under the NTK regime.} Theorem \ref{thm:solution} reveals the regularization effect of the KL divergence on the convergence of over-parameterized VAEs and makes a connection between solution of training a VAE and kernel ridge regression.

\section{Proof Sketch} \label{sec:proof}

In this section, we outline the approach used to establish the convergence results for VAEs and provide proofs for Theorem \ref{thm:opt} and Theorem \ref{thm:solution}. Our first step involves demonstrating that the NTKs, in the infinite-width limit, converge to deterministic kernels:
\begin{lemma}  \label{lem:init} Consider a stochastic network of the form (\ref{eq:net}), with the initialization of ${w}^{(\mu)}_{ij}  \sim \mathcal{N}(0, 1)$, ${w}^{(\sigma)}_{ij} = \sigma_0 $, and ${w}^{(d)}_{ij}  \sim \mathcal{N}(0, 1)$. Then the tangent kernels at initialization before training in the infinite-width limit follow the expression:
{\small \begin{equation}
{\small \begin{aligned}
\label{eq:limit_ntk}
 & \lim_{m \rightarrow \infty} \boldsymbol{\Theta}^{(\mu)}_{ij}(0)  
  =   \mathbb{E}_{{\bf w} } \big[{{\bf x}_i^{(e)}}^\top {\bf x}^{(e)}_j [\hat{\psi}'  \hat{\sigma}'( {\bf w}^\top {\bf x}^{(e)}_i)] [\hat{\psi}' \hat{\sigma}'( \mathbf{w}^\top {\bf x}^{(e)}_j)] \big] \otimes {\bf I}_{d \times d}, \\
 & \lim_{m \rightarrow \infty} \boldsymbol{\Theta}^{(\sigma)}_{ij}(0) 
 =    \mathbb{E}_{{\bf w} } \big[ [ \hat{\psi}' \hat{\mathbf{x}}^\top_i  \hat{\sigma}'( {\bf w}^\top {\bf x}^{(e)}_i)] [\hat{\psi}'\hat{\mathbf{x}}^{(e)}_j \hat{\sigma}'( \mathbf{w}^\top {\bf x}^{(e)}_j) ]\big] \otimes {\bf I}_{d \times d}, \\
&  \lim_{m \rightarrow \infty}  \boldsymbol{\Theta}^{(d)}_{ij}(0)   
 =  \mathbb{E}_{\bf w}  \big[ [ {\psi}(\hat{\sigma} ( {\bf w}^\top {\bf x}^{(e)}_i))] [ {\psi} (\hat{\sigma} ( {\bf w}^\top {\bf x}^{(e)}_j))] \big]\otimes {\bf I}_{d \times d}, \\
\end{aligned} }
\end{equation} }
where $\mathbf{w} \sim \mathcal{N}( \mathbf{0}, \mathbf{I})$ and we define:
{\small \begin{align*}
   & [\hat{\psi}' \hat{\sigma}'( \mathbf{w}^\top \mathbf{x}^{(e)}_i ) ]  \triangleq \mathbb{E}_{\boldsymbol{\zeta}} [  {\psi}'  {\sigma}'( (\mathbf{w} +  \sigma_0 \boldsymbol{\zeta})^\top \mathbf{x}^{(e)}_i ) ],\\
 &   [\hat{\psi}' \hat{\mathbf{x}}^\top_i \hat{\sigma}'( \mathbf{w}^\top \mathbf{x}^{(e)}_i ) ]   \triangleq \mathbb{E}_{\boldsymbol{\zeta}} [ {\psi}' (\mathbf{x}_i \odot \boldsymbol{\zeta})^\top {\sigma}'( (\mathbf{w} +  \sigma_0 \boldsymbol{\zeta})^\top \mathbf{x}^{(e)}_i ) ], \\
 &    [ {\psi}(\hat{\sigma}( \mathbf{w}^\top  \mathbf{x}^{(e)}_i )) ]  \triangleq \mathbb{E}_{\boldsymbol{\zeta}} [  {\psi}( {\sigma}( (\mathbf{w} +  \sigma_0 \boldsymbol{\zeta})^\top \mathbf{x}^{(e)}_i ) ]. 
\end{align*} }  
\end{lemma}

\begin{proof}[Proof of Lemma \ref{lem:init}] We first rewrite the expression for the stochastic neural network as follows:
{\small \begin{align*}
    \hat{\mathbf{f}}({\bf x}) = \mathbb{E}_{\boldsymbol{\zeta}} \left[\frac{1}{\sqrt{m}} \sum_{r=1}^m ({\bf w}_r^{(d)}) \psi( \sigma(({\bf w}_r^{(\mu)} + {\bf w}_r^{(\sigma)} \odot \boldsymbol{\zeta}_r)^\top {\bf x}^{(e)} )) \right].
\end{align*} }
Then the derivative of output function $\hat{f}_k(\mathbf{x}_i)$ for $k \in [1,d]$ with respect to the parameters ${\bf w}_r^{(\mu)}$, ${\bf w}_r^{(\sigma)}$ and ${\bf w}_r^{(d)}$ for $r \in [1,m]$ can be expressed as:
$$
{\small \begin{aligned}
\frac{\partial \hat{f}_k({\bf x}_i)}{\partial {\bf w}^{(\mu)}_r} &  = \mathbb{E}_{\boldsymbol{\zeta}_r} \left[ \frac{1}{\sqrt{m}} w^{(d)}_{r,k} \psi' \sigma'(z_{i,r}) {\bf x}_i \right], \\
\frac{\partial \hat{f}_k({\bf x}_i)}{\partial {\bf w}^{(\sigma)}_r } &= \mathbb{E}_{\boldsymbol{\zeta}_r} \left[ \frac{1}{\sqrt{m}} w^{(d)}_{r,k} \psi' \sigma'(z_{i,r}) {\bf x}_i \odot \boldsymbol{\zeta}_r \right], \\
\frac{\partial \hat{f}_k({\bf x}_i)}{\partial {\bf w}^{(d)}_r } &= \mathbb{E}_{\boldsymbol{\zeta}_r} \left[  \frac{1}{\sqrt{m}} \psi(\sigma (z_{i,r})) \boldsymbol{\delta}_{k} \right], 
\end{aligned} }
$$
where we have interchanged integration and differentiation over activation $\sigma(\cdot)$ by Assumption 4.2, and $\boldsymbol{\delta}_{k} \triangleq  [\delta_{1,k}, \delta_{2,k}, \cdots, \delta_{d,k} ]^\top \in \mathbb{R}^d $. We then calculate each NTK at initialization, i.e. $t=0$:

(1) The neural tangent kernel $\boldsymbol{\Theta}^{(\mu)}(0)$.
{\small \begin{align*}
\boldsymbol{\Theta}^{(\mu)}_{ik,jk'}(0) & = \frac{({\bf x} ^{(e)}_i)^\top {\bf x}^{(e)}_j}{m}  
  \sum_{r=1}^m \hat{\psi}' \hat{\sigma}'( z_{i,r})      \hat{\psi}' \hat{\sigma}'( {z}_{j,r})  \left( w^{(d)}_{r,k} w^{(d)}_{r,k'} \right),
\end{align*} }
where we define $\hat{\psi}'\hat \sigma'(z_{i,r}) \triangleq \mathbb{E}_{\boldsymbol{\zeta}_r} \left[ \psi' \sigma'( z_{i,r}) \right]$ and $z_{i,r} = \langle \mathbf{w}^{\mu}_r  + \mathbf{w}^{\sigma}_r \odot \boldsymbol{\zeta}_r , \mathbf{x}^{(e)}_i \rangle$.
For all pairs of $i,j,k,k'$, $\boldsymbol{\Theta}^{(\mu)}_{ik,jk'}(0)$ is the average of $m$ i.i.d. random variables. Because $w^{(d)}_{r,k}$ is i.i.d., we know that $\mathbb{E}\left[({ w}^{(d)}_{r,k}) ({ w}^{(d)}_{r,k'}) \right] = 0$. Therefore, we have
$$
\lim_{m \rightarrow \infty} \boldsymbol{\Theta}^{(\mu)}(0)  = \lim_{m \rightarrow \infty} \boldsymbol{\Theta}^{(\mu)}_{ij}(0) \otimes {\bf I}_{d \times d}.
$$
As a result, we conclude the proof:
{\small \begin{align*}
     & \lim_{m \rightarrow \infty} \boldsymbol{\Theta}^{(\mu)}_{ij}(0)  
  =   \mathbb{E}_{{\bf w} } \big[{{\bf x}_i^{(e)}}^\top {\bf x}^{(e)}_j [\hat{\psi}'  \hat{\sigma}'( {\bf w}^\top {\bf x}^{(e)}_i)] [\hat{\psi}' \hat{\sigma}'( \mathbf{w}^\top {\bf x}^{(e)}_j)] \big] \otimes {\bf I}_{d \times d}.
\end{align*} }

(2) Similarly, the neural tangent kernel $\boldsymbol{\Theta}^{(\sigma)}(0)$:
$$
\lim_{m \rightarrow \infty} \boldsymbol{\Theta}^{(\sigma)}(0)  = \lim_{m \rightarrow \infty} \boldsymbol{\Theta}^{(\sigma)}_{ij}(0) \otimes {\bf I}_{d \times d}.
$$

(3)  The neural tangent kernel $\boldsymbol{\Theta}^{(d)}(0)$.
{\small \begin{align*}
    \boldsymbol{\Theta}^{(d)}_{ij,kk'} = \frac{1}{m} \sum_{r=1}^m  \psi(\sigma( z_{i,r} )) \psi(\sigma( z_{j,r} )) \delta_{kk'}. 
\end{align*} }
Again, this neural tangent kernel is the average of $m$ i.i.d. random variables. Therefore we have $\lim_{m \rightarrow \infty}  \boldsymbol{\Theta}^{(d)}_{ij}(0) = \mathbb{E}_{\bf w}  \big[[ {\psi}(\hat{\sigma} ( {\bf w}^\top {\bf x}^{(e)}_i))] [ {\psi} (\hat{\sigma} ( {\bf w}^\top {\bf x}^{(e)}_j))] \big]$.
\end{proof}

Lemma \ref{lem:init} establishes that the NTKs converge to deterministic kernels in the infinite-width limit. We then study the behavior of tangent kernels with ultra-wide condition, namely $m = {\rm ploy}(n, 1/\lambda_0,1/\delta)$ \textit{at initialization}. The following lemma demonstrates that if $m$ is large, then $\boldsymbol{\Theta}^{(\mu)}(0)$, $\boldsymbol{\Theta}^{(\sigma)}(0)$, and  $\boldsymbol{\Theta}^{(d)}(0)$ have a lower bound on smallest eigenvalue with high probability.

\begin{lemma} [NTK at initialization] \label{lem:ntk_init}
If $m = \Omega\left(\frac{n^2d^2}{\lambda_0}\log\frac{nd}{\delta}\right)$, while ${  w}^{(\mu)}_{ij}$, ${ w}^{(\sigma)}_{ij}$, and ${ w}^{(d)}_{ij}$ are initialized by the form in Lemma \ref{lem:init}, then with probability at least $1-\delta$ over the initialization of weights, we have,
{\small \begin{equation}
{\small \begin{aligned}
   &  \left \|\boldsymbol{\Theta}^{(\mu)}(0) +\boldsymbol{\Theta}^{(\sigma)}( 0) + \boldsymbol{\Theta}^{(d)}( 0)   -\boldsymbol{\Theta}^\infty \right \|_2  \le \lambda_0/4, \\
   &  \left \|\boldsymbol{\Theta}^{(\mu)}( 0) +\boldsymbol{\Theta}^{(\sigma)}(  0) + \boldsymbol{\Theta}^{(d)}(  0) \right \|_2   \ge  3\lambda_0/{4}.
    \end{aligned} }
\end{equation} }
\end{lemma}
\begin{proof}[Proof of Lemma \ref{lem:ntk_init}]
The proof is by the standard concentration bound. By Lemma \ref{lem:init} we have shown that each neural tangent kernel is a sum of $m$ i.i.d. random variables. Then by Hoeffding's inequality for sub-Gaussian variable, we know that
$$
\left|\boldsymbol{\Theta}^{(\mu)}_{ik,jk'}(0) - \lim_{m \rightarrow \infty} \boldsymbol{\Theta}^{(\mu)}_{ik,jk'} \right| \le \sqrt{\frac{\log(2/\delta')}{2m}}
$$
holds with probability at least $(1-\delta')$. Because NTK matrix is of size $nd \times nd$, we then apply a union bound over all $i,j \in [n]$ and $k,k' \in [d]$. By setting $\delta'=\delta/(n^2d^2)$, we obtain that 
$$
\left|\boldsymbol{\Theta}^{(\mu)}_{ik,jk'}(0) - \lim_{m \rightarrow \infty} \boldsymbol{\Theta}^{(\mu)}_{ik,jk'} \right| \le \sqrt{\frac{\log(2n^2 d^2/\delta)}{2m}}.
$$
There by matrix perturbation theory we have,

{\small \begin{align*}
 \left \| \boldsymbol{\Theta}^{(\mu)}(0) - \lim_{m \rightarrow \infty} \boldsymbol{\Theta}^{(\mu)} \right \|^2_2  & \le \left \| \boldsymbol{\Theta}^\mu(0)- \lim_{m \rightarrow \infty} \boldsymbol{\Theta}^\infty \right \|^2_F  \le \sum_{i,j,k,k'} \left | \boldsymbol{\Theta}^{(\mu)}_{ik,jk'}(0) - \lim_{m \rightarrow \infty} \boldsymbol{\Theta}^{(\mu)}_{ik,jk'} \right |^2  \\
& = O \left(\frac{n^2 d^2 \log(nd/\delta)}{m} \right).
\end{align*} }
Similarly, applying the above argument to $\boldsymbol{\Theta}^{(\sigma)}$ and $\boldsymbol{\Theta}^{(d)}$ can yield the same result without much revision. Thus, by Hoeffding's inequality and union bound over matrix size, we know that the following inequalities hold with probability at least $(1-\delta)$,
$$
{\small \begin{aligned}
\left\| \boldsymbol{\Theta}^{(\sigma)}(0) - \lim_{m \rightarrow \infty} \boldsymbol{\Theta}^{(\sigma)} \right \|^2_2 & \le  O \left(\frac{n^2d^2 \log(nd/\delta)}{m} \right), \\ 
\left \| \boldsymbol{\Theta}^{(d)}(0) - \lim_{m \rightarrow \infty} \boldsymbol{\Theta}^{(d)} \right \|^2_2 & \le  O \left(\frac{n^2d^2 \log(nd/\delta)}{m} \right).
\end{aligned} }
$$
Finally, by the triangle inequality, we arrive at:
$$
{\small \begin{aligned}
& \left \| \boldsymbol{\Theta}^{(\mu)}(0)+  \boldsymbol{\Theta}^{(\sigma)}(0)+  \boldsymbol{\Theta}^{(d)}(0) -  \boldsymbol{\Theta}^{(\infty)} \right \|_2    \le \frac{\lambda_0}{4}.
\end{aligned} }
$$
On the other hand, we can achieve the lower bound by triangle inequality:
$$
{\small \begin{aligned}
& \left \| \boldsymbol{\Theta}^{(\mu)}(0)+  \boldsymbol{\Theta}^{(\sigma)}(0)+  \boldsymbol{\Theta}^{(d)}(0) \right \|_2   \ge \left \| \boldsymbol{\Theta}^{(\infty)} \right \|_2   - \left \| \boldsymbol{\Theta}^{(\mu)}(0)  +  \boldsymbol{\Theta}^{(\sigma)}(0)+  \boldsymbol{\Theta}^{(d)}(0) -  \boldsymbol{\Theta}^{(\infty)} \right \|_2  \ge \frac{3\lambda_0}{4 }.
\end{aligned} }
$$
We finalize the proof by setting $m = \Omega \left(\frac{n^2 d^2}{\lambda_0}\log\frac{nd}{\delta} \right)$.
\end{proof}

Lemma \ref{lem:ntk_init} completes the first step of our proof strategy, which states that if the width $m$ is large enough, then the neural tangent kernel of SNN \textit{at initialization} before training is close to the limiting kernel and is positive definite. 

However, a challenge arises due to the time-dependent nature of NTKs. These matrices evolve during the gradient descent training process. To account for this problem, we build a lemma stating that if the weights \textit{during training} are close to their initialization, then the NTKs \textit{during training} are close to the deterministic kernel $\boldsymbol{\Theta}^{(\infty)}$. Moreover, these NTKs will maintain a lower bound on their smallest eigenvalue, throughout the gradient descent training:
\begin{lemma}  \label{lem:train}
Suppose that $\| \mathbf{x}^{(e)}_i \|_2 = 1$, and at initialization that $\left \|{\bf W}^{(\mu)}(0) \right \|_F \le c_{\mu,2} \sqrt{m} $, $\left \|{\bf W}^{(\sigma)}(0) \right \|_F \le c_{\sigma,2} \sqrt{m} $, $\left \|{\bf w}_k^{(d)}(0) \right \|_2 \le c_{d,2} \sqrt{m} $, and $\left \|{\bf w}_k^{(d)}(0) \right\|_4 \le c_{d,4} m^{1/4} $ for $k \in [d]$.  If the weights at a training step $t$ satisfy: $ \left \| {\bf w}^{(\mu)}_{r}(t) - {\bf w}^{(\mu)}_{r}(0) \right \|_2 \triangleq R_{\mu} \le \frac{c_1 \lambda_0 }{n \sqrt{d}}$, $ \left \| {\bf w}^{(\sigma)}_{r}(t) - {\bf w}^{(\sigma)}_{r}(0) \right \|_2 \triangleq R_{\sigma} \le \frac{c_1 \lambda_0 }{n \sqrt{d}}$, and $ \left \| {\bf w}^{(d)}_{r}(t) - {\bf w}^{(d)}_{r}(0) \right \|_2 \le R_d \triangleq \frac{c_3 \lambda_0  }{n \sqrt{d}} $, where $c_1$, $c_2$, and $c_3$ are constants, then with probability at least $1-\delta$ over the random initialization, we have
{\small \begin{equation}
{\small \begin{aligned}
   &  \left \|\boldsymbol{\Theta}^{(\mu)}(t) +\boldsymbol{\Theta}^{(\sigma)}(  t) + \boldsymbol{\Theta}^{(d)}( t)  -\boldsymbol{\Theta}^\infty \right \|_2  \le  {\lambda_0}/{2}, 
   \left \|\boldsymbol{\Theta}^{(\mu)}( t) +\boldsymbol{\Theta}^{(\sigma)}(  t) + \boldsymbol{\Theta}^{(d)}( t) \right \|_2   \ge  {\lambda_0}/{2}.
    \end{aligned} }
\end{equation} }
\end{lemma}

\begin{proof} 
(1) We first analyze $\boldsymbol{\Theta}^{(\mu)}(t)$:
{\small \begin{align*}
 & \boldsymbol{\Theta}^{(\mu)}_{ik,jk'}(t)    = \frac{({\bf x}^{(e)}_i)^\top {\bf x}^{(e)}_j}{m} \sum_{r=1}^m \hat{\psi}' \hat{\sigma}'(z_{i,r} )   \hat{\psi}'\hat{\sigma}'(z_{j,r} )  w^{(d)}_{r,k}(t) w^{(d)}_{r,k'}(t).
\end{align*} }
Now we bound the distance between $\boldsymbol{\Theta}^{(\mu)}_{ik,jk}(t)$ and $\boldsymbol{\Theta}^{(\mu)}_{ik,jk}(0)$ through the following inequality:
{\small \begin{align*}
& \quad  \bigg | \boldsymbol{\Theta}_{ik,jk}^{(\mu)}(t)-\boldsymbol{\Theta}_{ik,jk}^{(\mu)}(0)\bigg | \\
	&\overset{(a)}{\le}  \frac{1}{m} \left| (\mathbf{x}^{(e)}_i)^\top \mathbf{x}^{(e)}_j  \right|    \Bigg | \sum_{r=1}^{m} (w^{(d)}_{r,k}(0))^2 [\hat{\psi}'\hat{\sigma}'\left(z_{i,r}(t)\right) \hat{\psi}' \hat{\sigma}'\left(z_{j,r}(t)\right)  - \hat{\psi}' \hat{\sigma}'\left(z_{i,r}(0)\right) \hat{\psi}' \hat{\sigma}'\left(z_{j,r}(0)\right)]
		\Bigg | \\
  & \quad +	\frac{1}{m} \left |  (\mathbf{x}^{(e)}_i)^\top \mathbf{x}^{(e)}_j	\right |  \left | \sum_{r=1}^{m} \left(w^{(d)}_{r,k}(t)^2-w^{(d)}_{r,k}(0)^2\right)
		\hat{\psi}' \hat{\sigma}'\left(z_{i,r}(t)\right) \hat{\psi}' \hat{\sigma}'\left(z_{j,r}(t)\right)
	\right | \\
& \overset{(b)}{\le}   \frac{2 \beta L^3}{m} \sum_{r=1}^{m} w^{(d)}_{r,k}(0)^2 \left \|\mathbf{w}^{(\mu)}_{r}(t) -
	\mathbf{w}^{(\mu)}_{r}(0)  \right \|_2   + \frac{2 \beta L^4}{m} \sum_{r=1}^{m} w^{(d)}_{r,k}(0)^2 \left \|\mathbf{w}^{(\mu)}_{r}(t) -
	\mathbf{w}^{(\mu)}_{r}(0)  \right \|_2 \\
	&  + \frac{L^4}{m}\sum_{r=1}^{m}	\bigg |w^{(d)}_{r,k}(t)^2-w^{(d)}_{r,k}(0)^2 	\bigg|  \le  \left( \frac{4 \beta L^3}{m} + \frac{4 \beta L^4}{m}\right) c^2_{d,4} \sqrt{m} R_\mu \sqrt{m}  + \frac{3L^4}{m} c_{d,2} R_d m.
	\end{align*} }
where $(a)$ is because of triangle inequality, and (b) is because of the assumptions that $\| \mathbf{x}^{(e)} \|_2 = 1$ as well as $L$-Lipschitz and $\beta$-Smooth of activations $\sigma(\cdot )$ and $\psi(\cdot)$. In particular, we have used the following inequalities:
{\small
\begin{align}
    \hat \sigma'(z_{i,r} (t)) - \hat \sigma'(z_{i,r} (0)) & = \mathbb{E}_{\boldsymbol{\zeta}_r} \left[ \sigma'(z_{i,r} (t))  - \sigma'(z_{i,r} (0))   \right]   \le  \beta \left \| \mathbf{w}^{(\mu)}_r(t) - \mathbf{w}^{(\mu)}_r(0) \right \|_2, \nonumber \\
     \hat{\psi}'(\sigma(z_{i,r} (t)) -  \hat{\psi}'(\sigma(z_{i,r} (0)) 
    & \le   \beta \mathbb{E}_{\boldsymbol{\zeta}_r} \left[ \sigma(z_{i,r} (t))  - \sigma(z_{i,r} (0))   \right]  
     \le     \beta L \left \| \mathbf{w}^{(\mu)}_r(t) - \mathbf{w}^{(\mu)}_r(0) \right \|_2. \nonumber
\end{align} }

Summing over all entries of the matrix, we can bound the perturbation:
{\small \begin{align*}
 \left \| \boldsymbol{\Theta}^{(\mu)}(t) - \boldsymbol{\Theta}^{(\mu)}(0) \right \|_2 &  \le  \sqrt{\sum_{i,j,k} \left| \boldsymbol{\Theta}_{ik,jk}^{(\mu)}(t) -   \boldsymbol{\Theta}_{ik,jk}^{(\mu)}(0)  \right|^2}  \\
  & \le  \left(4 \beta (L^3+L^4) c^2_{d,4} R_\mu  + 3L^4 c_{d,2} R_d \right) n \sqrt{d}.
\end{align*} }
Finally, due to the condition that $R_u \le \frac{c_1 \lambda_0 }{n \sqrt{d}}$ and $R_d \le \frac{c_3 \lambda_0 }{n \sqrt{d}}$, we have,
{\small \begin{align*}
 \left \|\boldsymbol{\Theta}^{(\mu)}(t) -\boldsymbol{\Theta}^{(\mu)}(0) \right \|_2 \le{\lambda_0}/{12}.   
\end{align*} }

(2) Similarly, 	
we have:
{\small \begin{align*}
\textcolor{black}{ \left \|\boldsymbol{\Theta}^{(\sigma)}(t) -\boldsymbol{\Theta}^{(\sigma)}(0) \right \|_2 \le{\lambda_0}/{12}.   }
\end{align*} }

(3) Finally,
	{\small \begin{align*}
    \bigg | \boldsymbol{\Theta}_{ik,jk}^{(d)}(t)-\boldsymbol{\Theta}_{ik,jk}^{(d)}(0)\bigg |  
&  {\le}   \frac{\beta L}{m} \sum_{r=1}^{m}  \left \|\mathbf{w}^{(\mu)}_{r}(t) -
	\mathbf{w}^{(\mu)}_{r}(0)  \right \|_2   \left \|\mathbf{w}^{(\mu)}_{r}(t) +
	\mathbf{w}^{(\mu)}_{r}(0)  \right \|_2   \\
&	 \le  \frac{2 \beta L}{m} (c_{\mu,2}+R_\mu) \sqrt{m} R_\mu \sqrt{m}   .
	\end{align*} }
 
With all the inequalities at hand, we conclude the proof:
$$
{\small \begin{aligned}
&\left \|\boldsymbol{\Theta}^{(\mu)}( t) +\boldsymbol{\Theta}^{(\sigma)}(  t) + \boldsymbol{\Theta}^{(d)}(  t)  -\boldsymbol{\Theta}^\infty  \right \|_2   \le \frac{3\lambda_0}{12} + \frac{\lambda_0}{4}  = \frac{\lambda_0}{2}.
\end{aligned} }
$$
\end{proof}

Lemma \ref{lem:train} demonstrates that if the change of weight is bounded, then the tangent kernel matrix is close to its expectation. The next lemma will show that the changes of weights \textit{during training} are bounded when the NTK is close to the limiting NTK:
\begin{lemma} \label{lem:muchange}
Suppose $\lambda_{0}(t) \ge \frac{\lambda_0}{2}$ for $0 < t < T $, then,

{\small \begin{equation}
{\small \begin{aligned}
    \left \| {\bf w}^{(s)}_r(t) - {\bf w}^{(s)}_r(0) \right \|_2 \le \frac{ \left \| {\bf X}-\hat{\mathbf{f}}({\bf X};0) \right \|_F \sqrt{n} d }{\sqrt{m}\lambda_0}  = R'_s, \text{where } s\in\{\mu,\sigma,d\}.\\
    \end{aligned} }
\end{equation} }

\end{lemma}

\begin{proof} [Proof of Lemma \ref{lem:muchange}]
The dynamics of loss can be calculated,
$$
{\small \begin{aligned}
& \frac{d}{dt}  \mathcal{L}(t) =  -\frac{1}{n} \left\| \left({\bf X}- \hat{\mathbf{f}} ({\bf X};t) \right)^\top  \boldsymbol{\Theta} (t)\left({\bf X}- \hat{\mathbf{f}} ({\bf X};t) \right)  \right \|_F   \le - \frac{\lambda_0}{n}  \left \| {\bf X} - \hat{\mathbf{f}}({\bf X};t) \right\|^2_F.
\end{aligned} }
$$
Integrating the differential function, the loss can be bounded as follows:
{\small \begin{align*}
\mathcal{L}(t)  \le \exp\left(-(\lambda_0/n) t \right) \mathcal{L}(0),
\end{align*} }
which implies the linear convergence rate of the stochastic neural network. Then the gradient flow for $\mathbf{w}_r^{(\mu)}$ is as follows,
$$
{\small \begin{aligned}
 \left \| \frac{d}{dt} \mathbf{w}^{(\mu)}_{r}(t) \right \|_2  & = \frac{1}{n} \left \| \sum_{i=1}^n  (\mathbf{x}_i - \hat{\mathbf{f}}_i(t))^\top  \frac{1}{\sqrt{m}}  \mathbf{w}^{(d)}_{r} \hat{\sigma}'( z_{i,r} ) \mathbf{x}_i \right \|_2 \\
&
 \le \frac{\beta}{n\sqrt{m}} \sum_{i=1}^n \left \|\mathbf{x}_i -  \hat{\mathbf{f}}_i(t)\right \|_2   \left \|\mathbf{w}^{(d)}_{r}(t) \right  \|_2 \left \|\mathbf{w}^{(\mu)}_{r}(t) \right  \|_2  \\
& \le \frac{\beta }{\sqrt{mn}}  \left \| \mathbf{X} -  \hat{\mathbf{f}}(\mathbf{X};0) \right \|_F \exp(- (\lambda_0/n) t)     \left(R'_d + \sqrt{d} c_{d,2} \right) \left(R'_\mu + \sqrt{d} c_{\mu,2} \right).
\end{aligned} }
$$
Integrating the gradient, we have: 
$$
{\small \begin{aligned}
& \left\|\mathbf{w}_{r}^{(\mu)}(T) -\mathbf{w}_{r}^{(\mu)}(0) \right \|_2 \le \int_{0}^T \left \|  \frac{d}{dt} \mathbf{w}_{r}^{(\mu)}(t) \right \|_2 dt    \le \frac{\beta \sqrt{n} \left \| \mathbf{X} - \hat{\mathbf{f}}({\bf X};0) \right \|_F  c_{d,2} c_{\mu,2} d } {\sqrt{m} \lambda_0}.
\end{aligned} }
$$

Similarly, we have: 
$$
{\small \begin{aligned}
& \left\|\mathbf{w}_{r}^{(d)}(T) -\mathbf{w}_{r}^{(d)}(0) \right \|_2 \le \int_{0}^T \left \|  \frac{d}{dt} \mathbf{w}_{r}^{(d)}(t) \right \|_2 dt    \le \frac{L \sqrt{n} \left \| \mathbf{X} - \hat{\mathbf{f}}({\bf X};0) \right \|_F c_{\mu,2} \sqrt{d} } {\sqrt{m} \lambda_0}.
\end{aligned} }
$$
\end{proof}

Lemma \ref{lem:muchange} states that once the least eigenvalue of NTK \textit{during training} are bounded, the change of weight will be bounded (evidenced by empirical simulation shown in Figure \ref{fig:weight}). By employing a proof by contradiction, combined with the results from Lemma, we can deduce that  \textit{during training} the NTKs of the SNN remain close to the deterministic kernel, provided the neural network is sufficiently wide. In a conclusion, with all the lemmas at hand, we arrive at the final Theorem \ref{thm:opt} by the following lemma:
\begin{lemma} \label{lem:final}
If $R'_{\mu} < R_{\mu}$, $R'_{\sigma} < R_{\sigma}$, and $R'_{d} < R_{d}$, then for all $t \ge 0$, $\lambda_{0}\left(\boldsymbol{\Theta} (t) \right) \ge \frac{\lambda_0}{2}$; Besides, the loss follows:
$$
 {L}(t)  \le \exp(-(\lambda_0/n) t)  {L}(0).
$$
\end{lemma}

\begin{proof} [Proof of Lemma \ref{lem:final}]
The proof is a standard contradiction. Suppose the conclusion does not hold at time $t$, which implies that there exists $r \in [m]$,  $ \left \| \mathbf{w}^{(\mu)}_r(t) -\mathbf{w}^{(\mu)}_r(t)  \right\|_2 > R' $,
then by Lemma \ref{lem:train} we know there exists $s \le t$ such that $ \lambda_0(\boldsymbol{\Theta}(s)) \le \lambda_0/2$. However, this is contradictory to Lemma \ref{lem:muchange}. 

To finalize the proof, we bound $\mathcal{L}(0)$:
$$
{\small \begin{aligned}
   &  \left\| {\bf X} -  \hat{\mathbf{f}} ({\bf X};0) \right\|^2_F  
     =  \sum_{i=1}^n   \left \| \textcolor{black}{\mathbf{x}^{(e)}_i }\right\|^2_2 + 2\left\| \textcolor{black}{\mathbf{x}^{(e)}_i} \right \|_2  \| \hat{\mathbf{f}}_i(0) \|_2  + \| \hat{\mathbf{f}}_i(0)\|_2^2   = \textcolor{black}{\Theta(n)}.
\end{aligned} }
$$
Finally, $R'_{\mu} < R_{\mu}$, $R'_{\sigma} < R_{\sigma}$, and $R'_{d} < R_{d}$ result in
$ m = \Omega \left( \frac{n^5 d^3  }{\lambda_0^4 \delta^2 } \right)$
which completes the proof.
\end{proof}

Finally, we give the detailed proof of Theorem \ref{thm:solution}, which is based on the linearization of the output function with respect to the weight space. 

\begin{proof} [Proof of Theorem \ref{thm:solution}]

Our proof first establishes the result of kernel ridge regression in the infinite-width limit, then bounds the perturbation on the network's prediction. The output function can be expressed as,
$$
{\small \begin{aligned}
 \hat{\mathbf{f}}^\infty ({\bf x};t) & = \hat{\mathbf{f}}^\infty ({\bf x};0)  + \boldsymbol{\Phi}_{{\mu}}({\bf x})^\top \left(\boldsymbol{\theta}^{(\mu)}(t)-\boldsymbol{\theta}^{(\mu)}(0)\right),
 \end{aligned} }
$$
where $\boldsymbol{\theta}^{(\mu)} \triangleq  \vec{ \mathbf{W} }^{(\mu)} \in \mathbb{R}^{md}$, and 
$\boldsymbol{\Phi}_{{\mu}}({\bf x}) = \nabla_{\boldsymbol{\theta}^{(\mu)}} \hat{\mathbf{f}}({\bf x},0) \in \mathbb{R}^{md \times d}$. It is known that the objective function with KL divergence follows:
{\small \begin{align*}
    \mathcal{L}(t) =   \frac{1}{2n}   \left \|\hat{\mathbf{f}} (\mathbf{X})- \mathbf{X} \right \|^2_F +  \beta    \left \| \boldsymbol{\theta}^{(\mu)}(t) - \boldsymbol{\theta}^{(\mu)}(0)  \right \|^2_2.
\end{align*} } 
We then calculate the gradient flow dynamics for mean weight:
$$
{\small \begin{aligned}
& \frac{d \boldsymbol{\theta}^{(\mu)}(t)}{d t}  = \frac{\partial \mathcal{L}(t)}{ \partial   \boldsymbol{\theta}^{(\mu)}}  =\boldsymbol{\Phi}_{{\mu}}({\bf X})  \left(\hat{\mathbf{f}}^\infty({\bf X};t)  -{\bf X} \right)  +  \beta \left(\boldsymbol{\theta}^{(\mu)}(t)-\boldsymbol{\theta}^{(\mu)}(0)\right)  \\ 
&= \boldsymbol{\Phi}_{{\mu}}({\bf X}) \boldsymbol{\Phi}_{{\mu}}({\bf X})^\top \left(\boldsymbol{\theta}^{(\mu)}(t)-  \boldsymbol{\theta}^{(\mu)}(0)\right)   +\boldsymbol{\Phi}_{{\mu}}({\bf X}) (\hat{\mathbf{f}}^\infty ({\bf X};0) - \mathbf{X})   + \beta  \left(\boldsymbol{\theta}^{(\mu)}(t)-\boldsymbol{\theta}^{(\mu)}(0)\right) \\
&  = \left( \boldsymbol{\Theta}^{{(\mu)}}  + \beta \mathbf{I} \right)  \left(\boldsymbol{\theta}^{(\mu)}(t)-\boldsymbol{\theta}^{(\mu)}(0)\right)   + \boldsymbol{\Phi}_{{\mu}}({\bf X}) \left(\hat{\mathbf{f}}^\infty ({\bf X};0) - \mathbf{X} \right),  
\end{aligned} }
$$
which is an ordinary differential equation. It is easy to see that the solution is,
$$
{\small \begin{aligned}
 \overline{\boldsymbol{\theta}}^{(\mu)}(t) & = \boldsymbol{\Phi}^\top_{\mu}(\mathbf{X})  \left(\boldsymbol{\Theta}^{(\mu)} +\beta {\bf I} \right)^{-1}   \left({\bf I}-e^{-(\boldsymbol{\Theta}^{(\mu)} +\beta {\bf I})t}\right)  \left(\hat{\mathbf{f}}^\infty ({\bf X};0) - \mathbf{X} \right),
\end{aligned} }
$$
where $\overline{\boldsymbol{\theta}}^{(\mu)}(t) \triangleq \boldsymbol{\theta}^{(\mu)}(t) - \boldsymbol{\theta}^{(\mu)}(0)$.
Plugging the result into the linearized output function, we have,
{\small \begin{align*}
\hat{\mathbf{f}}^\infty({\bf X};t) & =  {\bf X}-e^{-(\boldsymbol{\Theta}^{(\mu)}({\bf X},{\bf X}) +\beta  {\bf I})t} \left(\hat{\mathbf{f}}^\infty ({\bf X};0) - \mathbf{X} \right).
\end{align*} }
For an arbitrary test data $\mathbf{x}_{te}$, we have,
{\small \begin{align*}
\hat{\mathbf{f}}^\infty({\bf x}_{te};t) & =  \boldsymbol{\Theta}^{(\mu)}({\bf x}_{te},{\bf X}) \left(\boldsymbol{\Theta}^{(\mu)} +\beta {\bf I}  \right)^{-1}  \left({\bf I}-e^{-(\boldsymbol{\Theta}^{(\mu)}({\bf X},{\bf X}) +\beta {\bf I})t}\right)\mathbf{X}.
\end{align*} }
when we take the time to be infinity, 
{\small \begin{equation} \label{eq:kernel_ridge}
{\small \begin{aligned}
\hat{\mathbf{f}}^\infty({\bf x}_{te}; \infty) & = \boldsymbol{\Theta}^{(\mu)}({\bf x}_{te},{\bf X}) \left(\boldsymbol{\Theta}^{(\mu)} +\beta  {\bf I}  \right)^{-1} \mathbf{X}  .
\end{aligned} }
\end{equation} }

The next step is to show the difference between finite-width neural network and infinitely-wide network:
{\small \begin{align*}
 \left| \hat{\mathbf{f}}({\bf x}_{te}) - \hat{\mathbf{f}}^\infty({\bf x}_{te}) \right| \le
O(\mathcal{E}).
\end{align*} }
 where $\mathcal{E} = \mathcal{E}_{\rm init} + \frac{\sqrt{n}\mathcal{E}_\Theta }{ \lambda_0 +  \beta  } $ with $\left \|\hat{\mathbf{f}} \left(\boldsymbol{\theta}(0),\mathbf{x}_{te}\right) \right \|_2 \le \mathcal{E}_{\rm init}$ and $\|\boldsymbol{\Theta}^\infty -\boldsymbol{\Theta}(t)  \|_2 \le \mathcal{E}_{\Theta}$. 
Note the expression in Equation~\eqref{eq:kernel_ridge} can be rewritten as $\hat{\mathbf{f}}^\infty(\mathbf{x}_{te}) =   \boldsymbol{\Phi}(\mathbf{x}_{te})^\top \boldsymbol{\beta} $ and the solution to this equation can be further written as the result of applying gradient flow on the following kernel ridge regression problem
{\small \begin{align*}
\min_{\boldsymbol{\beta}} \sum_{i=1}^{n} \frac{1}{2n} \left \|   \boldsymbol{\Phi}(\mathbf{x}_i)^\top \boldsymbol{\beta} -\mathbf{x}_i\right \|_2^2 +  \beta  \left \| \boldsymbol{\beta} \right \|^2_2,
\end{align*} }
with initialization $\boldsymbol{\beta} (0) = 0$.
We use $\boldsymbol{\beta}(t)$ to denote this parameter at time $t$ trained by gradient flow and $\hat{\mathbf{f}}^\infty \left(\mathbf{x}_{te},\boldsymbol{\beta} (t)\right)$ be the predictor for $\mathbf{x}_{te}$ at time $t$.
With these notations, we rewrite {\small \begin{align*}
\hat{\mathbf{f}}^\infty (\mathbf{x}_{te}) = \int_{t=0}^{\infty}\frac{d \hat{\mathbf{f}}(\boldsymbol{\beta}(t),\mathbf{x}_{te})}{dt} dt,
\end{align*} }
where we have used the fact that the initial prediction is $0$.

We thus can analyze the difference between the SNN predictor and infinite-width SNN predictor via this integral form as follows:
{\small \begin{align*}
& \left \| \hat{\mathbf{f}}^\infty (\mathbf{x}_{te}) - \hat{\mathbf{f}} \left(\mathbf{x}_{te}\right) \right \|_2
\le   \left \| \hat{\mathbf{f}}(\boldsymbol{\theta}(0),\mathbf{x}_{te}) \right \|_2    + \left \| \int_{t =0}^{\infty} \left(\frac{d \hat{\mathbf{f}}(\boldsymbol{\theta}(t),\mathbf{x}_{te})}{dt}  - \frac{d \hat{\mathbf{f}}^\infty(\boldsymbol{\beta}(t),\mathbf{x}_{te})}{dt}\right) dt \right\|_2  \\
\le & \mathcal{E}_{init} + \bigg \| \frac{1}{n}  \int_{t=0}^\infty \left(\boldsymbol{\Theta}(\mathbf{x}_{te},\mathbf{X};t)-\boldsymbol{\Theta}^\infty(\mathbf{x}_{te},\mathbf{X})\right)^\top (\hat{\mathbf{f}}(t)-\mathbf{X})dt \\ &+ \beta    \int_{t=0}^\infty \left(\boldsymbol{\Phi}(\mathbf{x}_{te},t)-\boldsymbol{\Phi}^\infty(\mathbf{x}_{te})\right)^\top   \boldsymbol{\beta}(t)dt \bigg \|_2 \\  
&+ \bigg \| \frac{1}{n} \int_{t=0}^\infty \boldsymbol{\Theta}^\infty(\mathbf{x}_{te},\mathbf{X})^\top   (\hat{\mathbf{f}}^\infty(t)-\hat{\mathbf{f}} (t)) dt    +  \beta   \int_{t=0}^\infty \left(\boldsymbol{\Phi}^\infty(\mathbf{x}_{te})\right)^\top   (\boldsymbol{\beta}(t) - \overline{\boldsymbol{\theta}}(t)) dt \bigg \|_2 
\\ 
\le & \mathcal{E}_{init}+ \bigg(  \max_{0\le t \le \infty} \left\|\boldsymbol{\Theta}(\mathbf{x}_{te},\mathbf{X};t)-\boldsymbol{\Theta}^\infty(\mathbf{x}_{te},\mathbf{X}) \right\|_2  \int_{t=0}^\infty \left\|\hat{\mathbf{f}}(t)-\mathbf{X} \right\|_2 dt  \\ & +  \beta  \max_{0\le t \le \infty} \left\|\boldsymbol{\Phi}(\mathbf{x}_{te};t)-\boldsymbol{\Phi}^\infty(\mathbf{x}_{te}) \right\|_2   \int_{t=0}^\infty \| \boldsymbol{\beta} \|_2 dt\bigg) \\+ 
& \bigg(  \max_{0\le t \le \infty} \|\boldsymbol{\Theta}^\infty(\mathbf{x}_{te},\mathbf{X})\|_2     \int_{t=0}^\infty \|\hat{\mathbf{f}}(t)-\hat{\mathbf{f}}^\infty(t) \|_2 dt   +  \beta  \max_{0\le t \le \infty} \left\| \boldsymbol{\Phi}^\infty(\mathbf{x}_{te}) \right\|_2 \int_{t=0}^\infty \| \boldsymbol{\beta}(t) - \overline{\boldsymbol{\theta}}(t) \|_2 dt \bigg) \\
\triangleq & \mathcal{E}_{init} + I_2 + I_3.
\end{align*} }
For the second term $I_2$, recall that $\|\boldsymbol{\Theta}^\infty(\mathbf{x}_{te},\mathbf{X} )-\boldsymbol{\Theta}(\mathbf{x}_{te},\mathbf{X}; t)\|_2 \le \frac{\lambda_0}{2}$ by Lemma \ref{lem:train}.
Besides, we know that $\| \hat{\mathbf{f}}(t)- \mathbf{X} \|^2_2 +  \beta  \| \overline{\boldsymbol{\theta}} \|^2_2 \le  \exp(- (\frac{  \lambda_0}{2} +  \beta ) t) \| \hat{\mathbf{f}}(0)-  \mathbf{X}\|^2_2 $.
Therefore, we can bound:
{\small \begin{align*}
& \int_{0}^\infty \|\hat{\mathbf{f}}(t) - \mathbf{X}\|_2 +  \beta  \| \overline{\boldsymbol{\theta}}(t)\|_2 dt  
\le   \int_{t=0}^\infty \exp(- (\frac{  \lambda_0}{2} +  \beta  ) t)( \|\hat{\mathbf{f}}(0) - \mathbf{X}\|_2  ) dt  
 =   O\left(\frac{\sqrt{n}}{  \lambda_0 +  \beta  }\right).
\end{align*} }
As a result, we have 
    $I_2 = O \left(\frac{\sqrt{n}\mathcal{E}_\Theta }{  \lambda_0 +  \beta  }\right)$.
To bound $I_3$, we have 
{\small \begin{align*}
&  \int_{0}^\infty \| \hat{\mathbf{f}}(t)-\hat{\mathbf{f}}^\infty(t)\|_2 +  \beta  \|  \boldsymbol{\beta} -\overline{\boldsymbol{\theta}} \|_2 dt 
 \\ & \le  \int_{0}^\infty \|\hat{\mathbf{f}}(t)- \mathbf{X} \|_2 +  \beta  \| \overline{\boldsymbol{\theta}} \|_2 dt   + \int_{0}^\infty \|\hat{\mathbf{f}}^\infty(t)-  \mathbf{X} \|_2 + \beta  \| {\boldsymbol{\beta}} \|_2 dt  
=   O\left(\frac{\sqrt{n}}{  \lambda_0 + \beta  }\right).
\end{align*} }

As a result, we have $
    I_3 = O \left(\frac{\sqrt{n}\mathcal{E}_\Theta }{ \lambda_0 +  \beta  }\right)$.
Lastly, we put things together and get  {\small \begin{align*}
& |\hat{\mathbf{f}}(t)-\hat{\mathbf{f}}^\infty(t)|  
  = 
O\left(\mathcal{E}_{init} + \mathcal{E}_{\Theta} \frac{\sqrt{n}}{\lambda_0+  \beta  } \right). 
\end{align*} }

\end{proof}

\vspace{-1cm}

\section{Experiments}

In this section, we provide empirical evidence to support our theoretical analysis concerning the training dynamics of over-parameterized stochastic neural networks, which are optimized using VAE training objectives. Our experimental results, derived from training on the MNIST dataset, corroborate our theoretical predictions. In addition, we report our observation that VAEs with larger latent spaces are capable of learning more information, which substantiates the rationale behind our theoretical examination of the convergence properties of over-parameterized VAEs.

\subsection{Theoretical verification}

\begin{figure}
  \begin{center}
\includegraphics[width=0.85\textwidth]{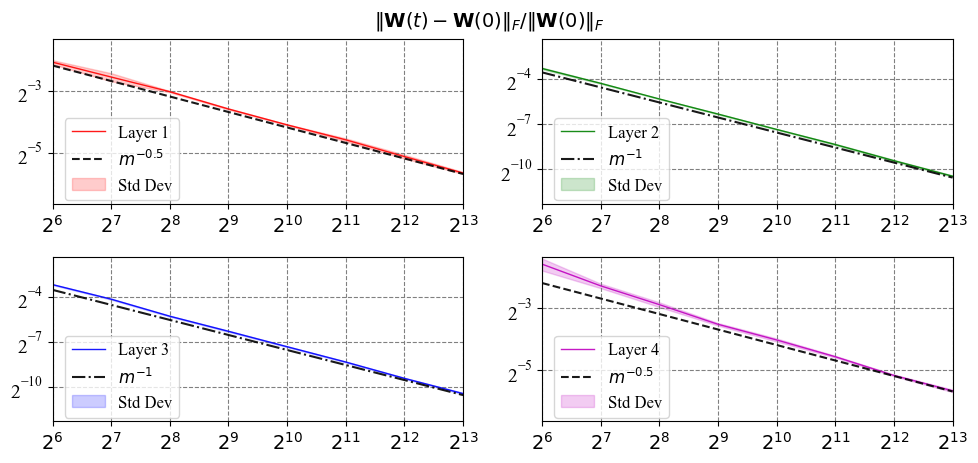} 
  \end{center}
    \caption{\textcolor{black}{Relative Frobenius norm change in weights after training, where $m$ is the width of the network. Solid lines correspond to empirical simulations and dotted lines are theoretical predictions.}}
    \label{fig:weight}
   \vspace{-3mm}
\end{figure}

To empirically validate our lemmas, we employ a three-hidden-layer \textcolor{black}{} fully connected network, guided by the training objective function as presented in Equation (\ref{eq:objective}). The network parameters are initialized using the Neural Tangent Kernel (NTK) parameterization, in line with Equation (\ref{eq:net}). For training, we adopt the ordinary mean-squared error (MSE) as the reconstruction loss and employ full-batch gradient descent with a consistent learning rate of 1  on a subset of the MNIST dataset containing 128 samples and 10 classes. We measure the change in weights of each layer, denoted by $\| \mathbf{W}(t)-\mathbf{W}(0) \|_F/ \|\mathbf{W}(0) \|_F$, after performing $t=2^{17}$ steps of gradient descent updates from random initialization. Figure \ref{fig:weight} displays the results for each layer. We only measure the change in weight $\mathbf{W}^{(\mu)}$ for the latent layer ($\mu$). Our observations show that the relative Frobenius norm changes in the Encoder and Decoder scales as $1/\sqrt{m}$, while the hidden layers' weights scale as $1/m$. This result confirms that the weights of SNN do not move too much during training, and further confirms the correctness of our theoretical claim (Lemma \ref{lem:muchange}). Notably, a similar convergence rate for weight changes in deterministic neural networks was observed in \cite{lee2019wide}. 
To empirically validate our lemmas, we employ a three-hidden-layer \textcolor{black}{Tanh} fully connected network, guided by the training objective function as presented in Equation (\ref{eq:objective}). The network parameters are initialized using the Neural Tangent Kernel (NTK) parameterization, in line with Equation (\ref{eq:net}). For training, we adopt the ordinary mean-squared error (MSE) as the reconstruction loss and employ full-batch gradient descent with a consistent learning rate of 1  on a subset of the MNIST dataset containing 128 samples and 10 classes. We measure the change in weights of each layer, denoted by $\| \mathbf{W}(t)-\mathbf{W}(0) \|_F/ \|\mathbf{W}(0) \|_F$, after performing $t=2^{17}$ steps of gradient descent updates from random initialization. Figure \ref{fig:weight} displays the results for each layer. We only measure the change in weight $\mathbf{W}^{(\mu)}$ for the latent layer ($\mu$). Our observations show that the relative Frobenius norm changes in the Encoder and Decoder scales as $1/\sqrt{m}$, while the hidden layers' weights scale as $1/m$. This result confirms that the weights of SNN do not move too much during training, and further confirms the correctness of our theoretical claim (Lemma \ref{lem:muchange}). Notably, a similar convergence rate for weight changes in deterministic neural networks was observed in \cite{lee2019wide}.

\subsection{Large latent space can learn more}
\begin{figure*}
    \centering
   \begin{subfigure}[t]{0.24\textwidth}
       \raisebox{-\height}{\includegraphics[width=\textwidth]{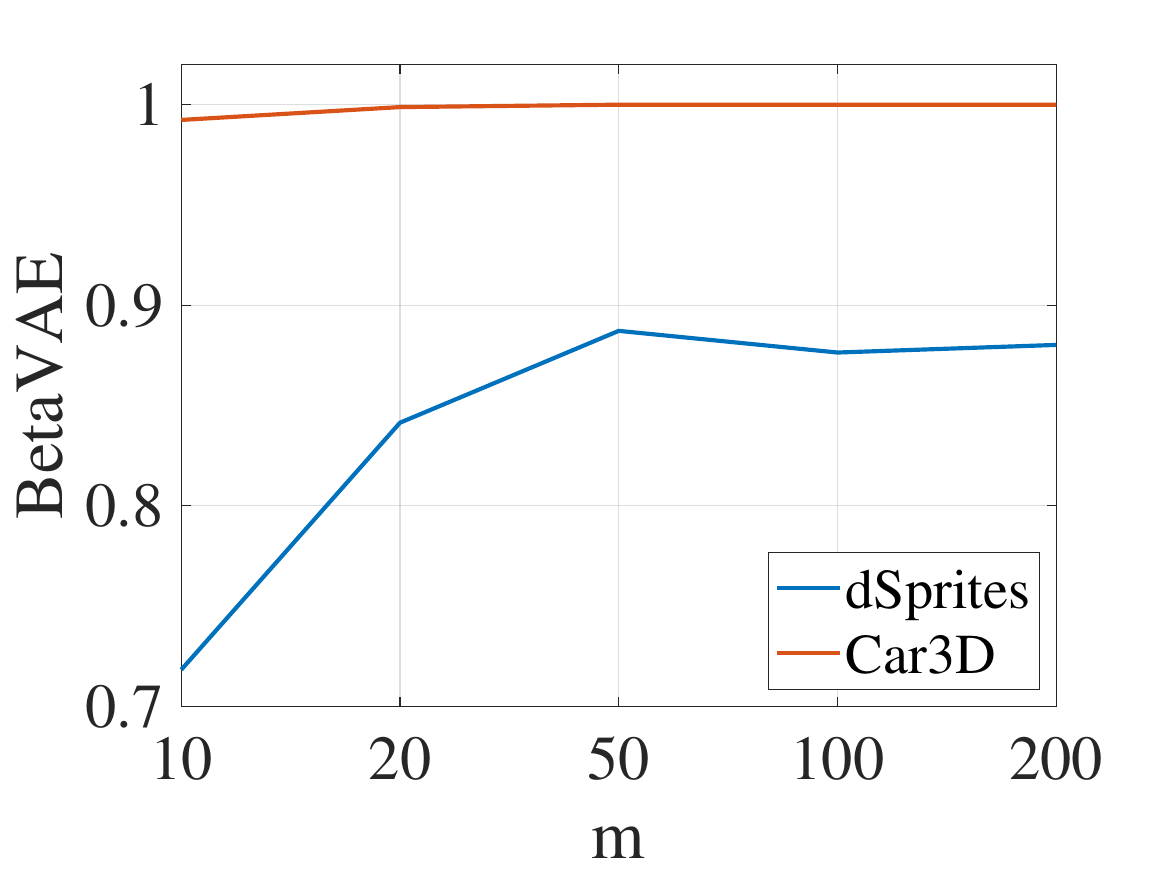}}
   \end{subfigure}
   \hfill
   \begin{subfigure}[t]{0.24\textwidth}
       \raisebox{-\height}{\includegraphics[width=\textwidth]{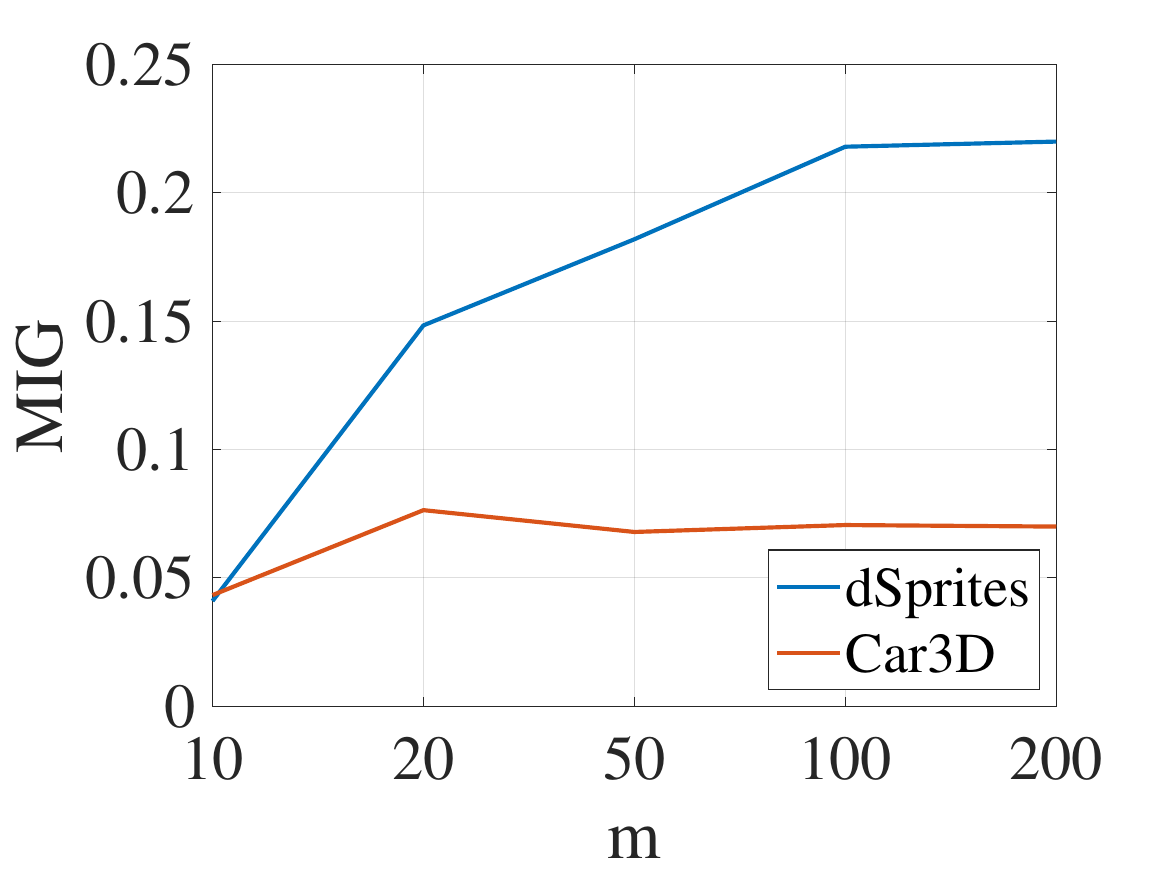}}
   \end{subfigure}
   \hfill
   \begin{subfigure}[t]{0.24\textwidth}
       \raisebox{-\height}{\includegraphics[width=\textwidth]{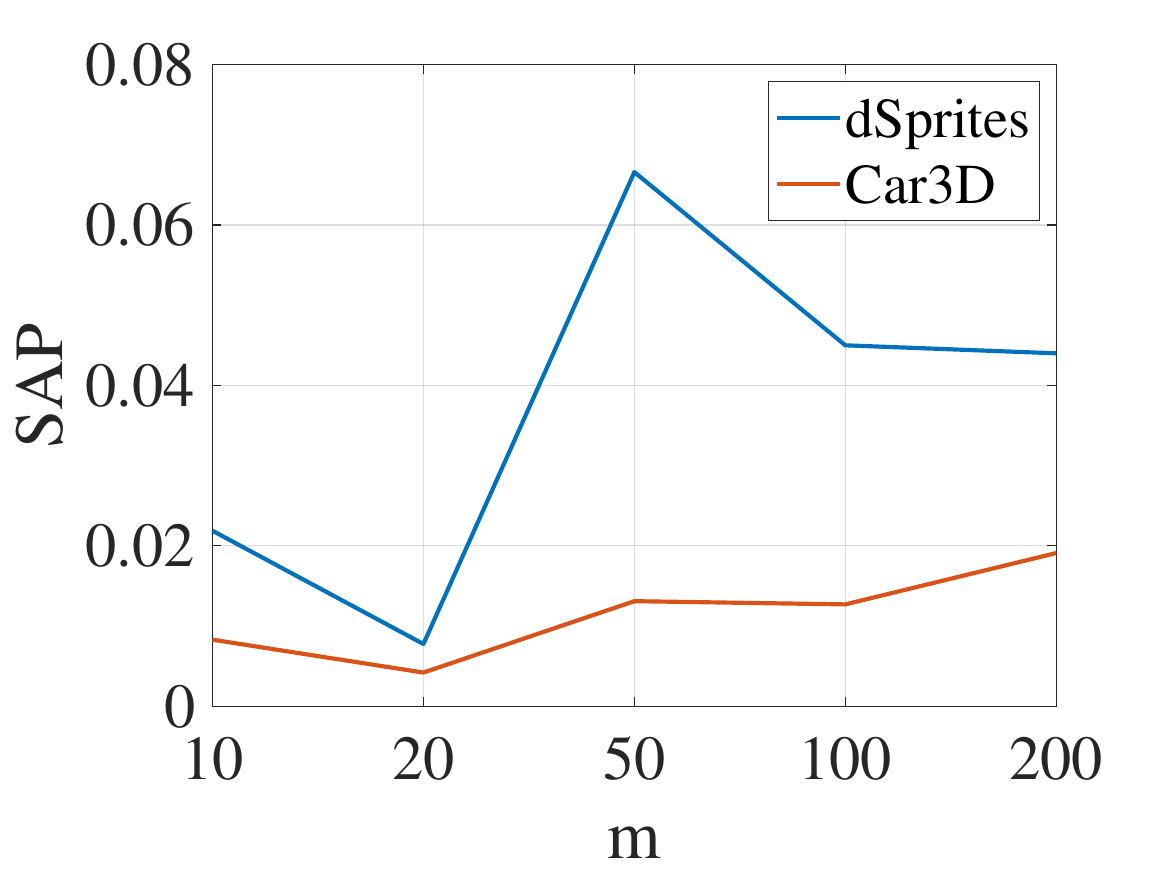}}
   \end{subfigure}
   \hfill
   \begin{subfigure}[t]{0.24\textwidth}
       \raisebox{-\height}{\includegraphics[width=\textwidth]{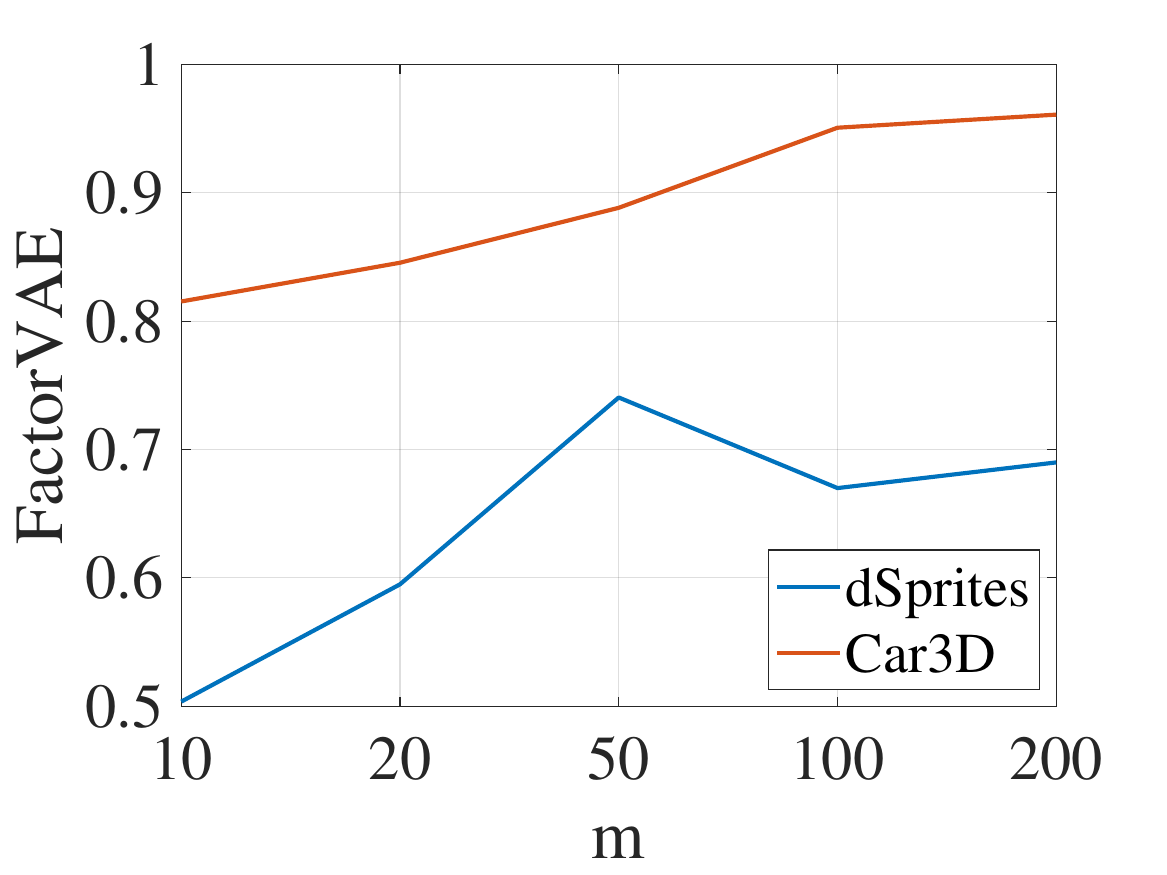}}
   \end{subfigure}
   \caption{Disentanglement scores for networks of latent dimension: $m=10,20,50,100,200$ on dSprites and Cars 3D. Observations: the larger the latent space, the better the disentangle learning. 
} \label{fig:measure}
\end{figure*}


In this subsection, we report our experimental observations, aligning with numerous prior studies \cite{song2019latent,lim2020deep}. We observed that larger latent spaces are capable of capturing more information, as evidenced by higher disentanglement scores and the emergence of additional features not discernible in models with narrower VAE configurations.

Adopting the experimental setup utilized in Beta-VAE\cite{higgins2016beta}, we explored the effects of varying latent space dimensions. Our experiments were conducted on the dSprites\cite{higgins2016beta} and Cars3D datasets\cite{reed2015deep}. As shown in Figure \ref{fig:measure}, the width of the latent space, denoted by $m$, is varied across [10,20,50,100,200]. We assessed the performance using a suite of disentanglement score metrics, including the BetaVAE, $\beta$-VAE metric \cite{higgins2016beta}, Mutual Information Gap (MIG) \cite{chen2615isolating}, Separated Attribute Predictability (SAP) score \cite{kumar2017variational}, and Factor-VAE metric \cite{kim2018disentangling}. Our findings indicate that larger latent spaces lead to higher disentanglement scores, with the exception of a less pronounced improvement when employing the BetaVAE metric on the Cars3D dataset. These results corroborate the hypothesis that a larger latent space is capable of capturing more information.

\begin{figure*}
    \centering
    \includegraphics[width=\textwidth]{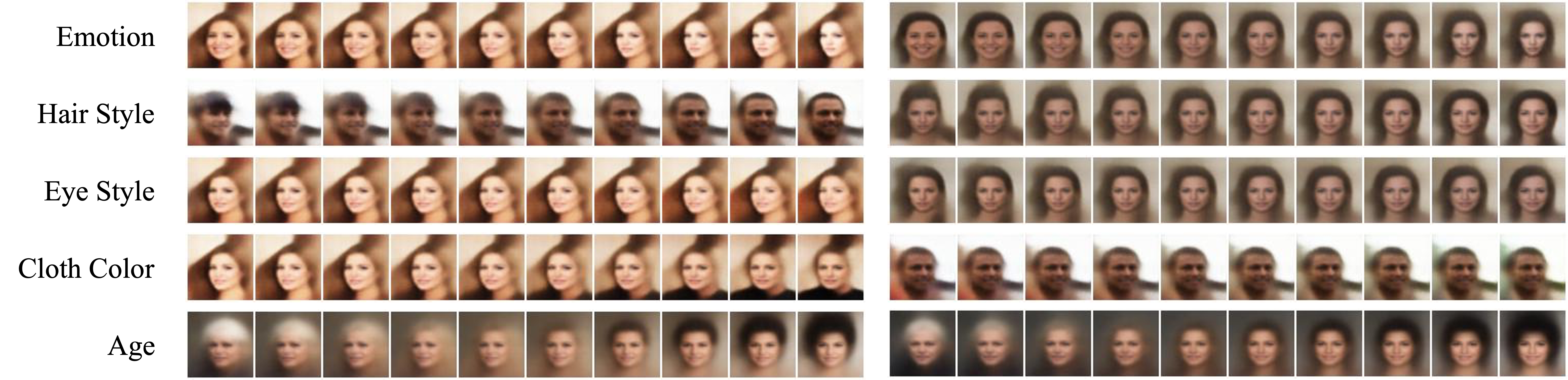}
    \caption{New image attributes discovered by large latent space VAE ($m=256$) but not by small latent space VAE ($m=10$) CelebA dataset.}
    \label{fig:add}
\end{figure*}

Furthermore, in our experiments with the CelebA \cite{liu2015deep} datasets, we observed that a larger latent space can reveal additional features not detected in smaller latent space VAEs. As illustrated in Figure \ref{fig:add}, on the CelebA dataset, a VAE with a latent space of 256 dimensions uncovered new image attributes such as emotion, eye style, and hairstyle, which were not identified by a VAE with a latent space of just 10 dimensions. 
These findings confirm that VAEs with larger latent spaces are capable of detecting additional features not observable in narrower VAE configurations.

These observations validate the intuitive notion that VAEs with larger latent spaces exhibit superior disentanglement performance. This underlines our initial motivation for investigating over-parameterized VAEs, as opposed to conventional VAEs, to leverage the benefits of increased latent dimensionality.

\vspace{-0.3cm}
\section{Conclusion}
\vspace{-0.2cm}
\textcolor{black}{In this work, we have established the convergence of over-parameterized VAEs using the neural tangent kernel techniques. Additionally, we have demonstrated that the expected output function trained with the full objective function and KL divergence converges to the kernel ridge regression, confirming the regularization effect of the additional KL divergence. The theoretical insights presented in this paper pave the way for analyzing stochastic neural networks within other paradigms, such as deep Bayesian networks. Our empirical evaluations corroborate that the theoretical predictions are consistent with real-world training dynamics. 
Furthermore, through experimental investigations on real datasets, we have highlighted the training efficiency of over-parameterized VAEs, as suggested by our theoretical findings. 
}

\def\bibfont{\small}
\vspace{-0.3cm}
\bibliography{sn-bibliography}


\begin{thebibliography}{50}
\ifx \bisbn   \undefined \def \bisbn  #1{ISBN #1}\fi
\ifx \binits  \undefined \def \binits#1{#1}\fi
\ifx \bauthor  \undefined \def \bauthor#1{#1}\fi
\ifx \batitle  \undefined \def \batitle#1{#1}\fi
\ifx \bjtitle  \undefined \def \bjtitle#1{#1}\fi
\ifx \bvolume  \undefined \def \bvolume#1{\textbf{#1}}\fi
\ifx \byear  \undefined \def \byear#1{#1}\fi
\ifx \bissue  \undefined \def \bissue#1{#1}\fi
\ifx \bfpage  \undefined \def \bfpage#1{#1}\fi
\ifx \blpage  \undefined \def \blpage #1{#1}\fi
\ifx \burl  \undefined \def \burl#1{\textsf{#1}}\fi
\ifx \doiurl  \undefined \def \doiurl#1{\url{https://doi.org/#1}}\fi
\ifx \betal  \undefined \def \betal{\textit{et al.}}\fi
\ifx \binstitute  \undefined \def \binstitute#1{#1}\fi
\ifx \binstitutionaled  \undefined \def \binstitutionaled#1{#1}\fi
\ifx \bctitle  \undefined \def \bctitle#1{#1}\fi
\ifx \beditor  \undefined \def \beditor#1{#1}\fi
\ifx \bpublisher  \undefined \def \bpublisher#1{#1}\fi
\ifx \bbtitle  \undefined \def \bbtitle#1{#1}\fi
\ifx \bedition  \undefined \def \bedition#1{#1}\fi
\ifx \bseriesno  \undefined \def \bseriesno#1{#1}\fi
\ifx \blocation  \undefined \def \blocation#1{#1}\fi
\ifx \bsertitle  \undefined \def \bsertitle#1{#1}\fi
\ifx \bsnm \undefined \def \bsnm#1{#1}\fi
\ifx \bsuffix \undefined \def \bsuffix#1{#1}\fi
\ifx \bparticle \undefined \def \bparticle#1{#1}\fi
\ifx \barticle \undefined \def \barticle#1{#1}\fi
\bibcommenthead
\ifx \bconfdate \undefined \def \bconfdate #1{#1}\fi
\ifx \botherref \undefined \def \botherref #1{#1}\fi
\ifx \url \undefined \def \url#1{\textsf{#1}}\fi
\ifx \bchapter \undefined \def \bchapter#1{#1}\fi
\ifx \bbook \undefined \def \bbook#1{#1}\fi
\ifx \bcomment \undefined \def \bcomment#1{#1}\fi
\ifx \oauthor \undefined \def \oauthor#1{#1}\fi
\ifx \citeauthoryear \undefined \def \citeauthoryear#1{#1}\fi
\ifx \endbibitem  \undefined \def \endbibitem {}\fi
\ifx \bconflocation  \undefined \def \bconflocation#1{#1}\fi
\ifx \arxivurl  \undefined \def \arxivurl#1{\textsf{#1}}\fi
\csname PreBibitemsHook\endcsname

\bibitem[\protect\citeauthoryear{Kingma and Welling}{2013}]{kingma2013auto}
\begin{botherref}
\oauthor{\bsnm{Kingma}, \binits{D.P.}},
\oauthor{\bsnm{Welling}, \binits{M.}}:
Auto-encoding variational bayes.
arXiv preprint arXiv:1312.6114
(2013)
\end{botherref}
\endbibitem

\bibitem[\protect\citeauthoryear{Radford et~al.}{2015}]{radford2015unsupervised}
\begin{botherref}
\oauthor{\bsnm{Radford}, \binits{A.}},
\oauthor{\bsnm{Metz}, \binits{L.}},
\oauthor{\bsnm{Chintala}, \binits{S.}}:
Unsupervised representation learning with deep convolutional generative adversarial networks.
arXiv preprint arXiv:1511.06434
(2015)
\end{botherref}
\endbibitem

\bibitem[\protect\citeauthoryear{Van Den~Oord and Vinyals}{2017}]{ref:vae-image}
\begin{bchapter}
\bauthor{\bsnm{Van Den~Oord}, \binits{A.}},
\bauthor{\bsnm{Vinyals}, \binits{O.}}:
\bctitle{Neural discrete representation learning}.
In: \bbtitle{Advances in Neural Information Processing Systems},
pp. \bfpage{6306}--\blpage{6315}
(\byear{2017})
\end{bchapter}
\endbibitem

\bibitem[\protect\citeauthoryear{Wang et~al.}{2022}]{wang2022pruning}
\begin{barticle}
\bauthor{\bsnm{Wang}, \binits{L.}},
\bauthor{\bsnm{Huang}, \binits{W.}},
\bauthor{\bsnm{Zhang}, \binits{M.}},
\bauthor{\bsnm{Pan}, \binits{S.}},
\bauthor{\bsnm{Chang}, \binits{X.}},
\bauthor{\bsnm{Su}, \binits{S.W.}}:
\batitle{Pruning graph neural networks by evaluating edge properties}.
\bjtitle{Knowledge-Based Systems}
\bvolume{256},
\bfpage{109847}
(\byear{2022})
\end{barticle}
\endbibitem

\bibitem[\protect\citeauthoryear{Bowman et~al.}{2015}]{bowman2015generating}
\begin{botherref}
\oauthor{\bsnm{Bowman}, \binits{S.R.}},
\oauthor{\bsnm{Vilnis}, \binits{L.}},
\oauthor{\bsnm{Vinyals}, \binits{O.}},
\oauthor{\bsnm{Dai}, \binits{A.M.}},
\oauthor{\bsnm{Jozefowicz}, \binits{R.}},
\oauthor{\bsnm{Bengio}, \binits{S.}}:
Generating sentences from a continuous space.
arXiv preprint arXiv:1511.06349
(2015)
\end{botherref}
\endbibitem

\bibitem[\protect\citeauthoryear{Ng et~al.}{2011}]{ng2011sparse}
\begin{barticle}
\bauthor{\bsnm{Ng}, \binits{A.}}, \betal:
\batitle{Sparse autoencoder}.
\bjtitle{CS294A Lecture notes}
\bvolume{72}(\bissue{2011}),
\bfpage{1}--\blpage{19}
(\byear{2011})
\end{barticle}
\endbibitem

\bibitem[\protect\citeauthoryear{Tschannen et~al.}{2018}]{tschannen2018recent}
\begin{botherref}
\oauthor{\bsnm{Tschannen}, \binits{M.}},
\oauthor{\bsnm{Bachem}, \binits{O.}},
\oauthor{\bsnm{Lucic}, \binits{M.}}:
Recent advances in autoencoder-based representation learning.
arXiv preprint arXiv:1812.05069
(2018)
\end{botherref}
\endbibitem

\bibitem[\protect\citeauthoryear{Song et~al.}{2019}]{song2019latent}
\begin{barticle}
\bauthor{\bsnm{Song}, \binits{T.}},
\bauthor{\bsnm{Sun}, \binits{J.}},
\bauthor{\bsnm{Chen}, \binits{B.}},
\bauthor{\bsnm{Peng}, \binits{W.}},
\bauthor{\bsnm{Song}, \binits{J.}}:
\batitle{Latent space expanded variational autoencoder for sentence generation}.
\bjtitle{IEEE Access}
\bvolume{7},
\bfpage{144618}--\blpage{144627}
(\byear{2019})
\end{barticle}
\endbibitem

\bibitem[\protect\citeauthoryear{Lim et~al.}{2020}]{lim2020deep}
\begin{barticle}
\bauthor{\bsnm{Lim}, \binits{K.-L.}},
\bauthor{\bsnm{Jiang}, \binits{X.}},
\bauthor{\bsnm{Yi}, \binits{C.}}:
\batitle{Deep clustering with variational autoencoder}.
\bjtitle{IEEE Signal Processing Letters}
\bvolume{27},
\bfpage{231}--\blpage{235}
(\byear{2020})
\end{barticle}
\endbibitem

\bibitem[\protect\citeauthoryear{Zhang et~al.}{2022}]{zhang2022weighted}
\begin{barticle}
\bauthor{\bsnm{Zhang}, \binits{M.}},
\bauthor{\bsnm{Wang}, \binits{L.}},
\bauthor{\bsnm{Campos}, \binits{D.}},
\bauthor{\bsnm{Huang}, \binits{W.}},
\bauthor{\bsnm{Guo}, \binits{C.}},
\bauthor{\bsnm{Yang}, \binits{B.}}:
\batitle{Weighted mutual learning with diversity-driven model compression}.
\bjtitle{Advances in Neural Information Processing Systems}
\bvolume{35},
\bfpage{11520}--\blpage{11533}
(\byear{2022})
\end{barticle}
\endbibitem

\bibitem[\protect\citeauthoryear{He et~al.}{2019}]{he2019lagging}
\begin{botherref}
\oauthor{\bsnm{He}, \binits{J.}},
\oauthor{\bsnm{Spokoyny}, \binits{D.}},
\oauthor{\bsnm{Neubig}, \binits{G.}},
\oauthor{\bsnm{Berg-Kirkpatrick}, \binits{T.}}:
Lagging inference networks and posterior collapse in variational autoencoders.
arXiv preprint arXiv:1901.05534
(2019)
\end{botherref}
\endbibitem

\bibitem[\protect\citeauthoryear{Lucas et~al.}{2019}]{lucas2019don}
\begin{botherref}
\oauthor{\bsnm{Lucas}, \binits{J.}},
\oauthor{\bsnm{Tucker}, \binits{G.}},
\oauthor{\bsnm{Grosse}, \binits{R.B.}},
\oauthor{\bsnm{Norouzi}, \binits{M.}}:
Don't blame the {ELBO}! a linear {VAE} perspective on posterior collapse.
Advances in Neural Information Processing Systems
\textbf{32}
(2019)
\end{botherref}
\endbibitem

\bibitem[\protect\citeauthoryear{Koehler et~al.}{2021}]{koehler2021variational}
\begin{botherref}
\oauthor{\bsnm{Koehler}, \binits{F.}},
\oauthor{\bsnm{Mehta}, \binits{V.}},
\oauthor{\bsnm{Risteski}, \binits{A.}},
\oauthor{\bsnm{Zhou}, \binits{C.}}:
Variational autoencoders in the presence of low-dimensional data: landscape and implicit bias.
arXiv preprint arXiv:2112.06868
(2021)
\end{botherref}
\endbibitem

\bibitem[\protect\citeauthoryear{Jacot et~al.}{2018}]{jacot2018neural}
\begin{botherref}
\oauthor{\bsnm{Jacot}, \binits{A.}},
\oauthor{\bsnm{Gabriel}, \binits{F.}},
\oauthor{\bsnm{Hongler}, \binits{C.}}:
Neural tangent kernel: Convergence and generalization in neural networks.
arXiv preprint arXiv:1806.07572
(2018)
\end{botherref}
\endbibitem

\bibitem[\protect\citeauthoryear{Allen-Zhu et~al.}{2019}]{allen2019convergence}
\begin{bchapter}
\bauthor{\bsnm{Allen-Zhu}, \binits{Z.}},
\bauthor{\bsnm{Li}, \binits{Y.}},
\bauthor{\bsnm{Song}, \binits{Z.}}:
\bctitle{A convergence theory for deep learning via over-parameterization}.
In: \bbtitle{International Conference on Machine Learning},
pp. \bfpage{242}--\blpage{252}
(\byear{2019}).
\bcomment{PMLR}
\end{bchapter}
\endbibitem

\bibitem[\protect\citeauthoryear{Du et~al.}{2018}]{du2018gradient}
\begin{botherref}
\oauthor{\bsnm{Du}, \binits{S.S.}},
\oauthor{\bsnm{Zhai}, \binits{X.}},
\oauthor{\bsnm{Poczos}, \binits{B.}},
\oauthor{\bsnm{Singh}, \binits{A.}}:
Gradient descent provably optimizes over-parameterized neural networks.
arXiv preprint arXiv:1810.02054
(2018)
\end{botherref}
\endbibitem

\bibitem[\protect\citeauthoryear{Du et~al.}{2019}]{du2019gradient}
\begin{bchapter}
\bauthor{\bsnm{Du}, \binits{S.}},
\bauthor{\bsnm{Lee}, \binits{J.}},
\bauthor{\bsnm{Li}, \binits{H.}},
\bauthor{\bsnm{Wang}, \binits{L.}},
\bauthor{\bsnm{Zhai}, \binits{X.}}:
\bctitle{Gradient descent finds global minima of deep neural networks}.
In: \bbtitle{International Conference on Machine Learning},
pp. \bfpage{1675}--\blpage{1685}
(\byear{2019}).
\bcomment{PMLR}
\end{bchapter}
\endbibitem

\bibitem[\protect\citeauthoryear{Huang et~al.}{2020}]{huang2020neural}
\begin{botherref}
\oauthor{\bsnm{Huang}, \binits{W.}},
\oauthor{\bsnm{Du}, \binits{W.}},
\oauthor{\bsnm{Da~Xu}, \binits{R.Y.}}:
On the neural tangent kernel of deep networks with orthogonal initialization.
arXiv preprint arXiv:2004.05867
(2020)
\end{botherref}
\endbibitem

\bibitem[\protect\citeauthoryear{Huang et~al.}{2021}]{huang2021towards}
\begin{botherref}
\oauthor{\bsnm{Huang}, \binits{W.}},
\oauthor{\bsnm{Li}, \binits{Y.}},
\oauthor{\bsnm{Du}, \binits{W.}},
\oauthor{\bsnm{Da~Xu}, \binits{R.Y.}},
\oauthor{\bsnm{Yin}, \binits{J.}},
\oauthor{\bsnm{Chen}, \binits{L.}},
\oauthor{\bsnm{Zhang}, \binits{M.}}:
Towards deepening graph neural networks: A gntk-based optimization perspective.
arXiv preprint arXiv:2103.03113
(2021)
\end{botherref}
\endbibitem

\bibitem[\protect\citeauthoryear{Zou et~al.}{2020}]{zou2020gradient}
\begin{barticle}
\bauthor{\bsnm{Zou}, \binits{D.}},
\bauthor{\bsnm{Cao}, \binits{Y.}},
\bauthor{\bsnm{Zhou}, \binits{D.}},
\bauthor{\bsnm{Gu}, \binits{Q.}}:
\batitle{Gradient descent optimizes over-parameterized deep relu networks}.
\bjtitle{Machine Learning}
\bvolume{109}(\bissue{3}),
\bfpage{467}--\blpage{492}
(\byear{2020})
\end{barticle}
\endbibitem

\bibitem[\protect\citeauthoryear{Chen et~al.}{2021}]{chen2021equivalence}
\begin{botherref}
\oauthor{\bsnm{Chen}, \binits{Y.}},
\oauthor{\bsnm{Huang}, \binits{W.}},
\oauthor{\bsnm{Nguyen}, \binits{L.}},
\oauthor{\bsnm{Weng}, \binits{T.-W.}}:
On the equivalence between neural network and support vector machine.
Advances in Neural Information Processing Systems
\textbf{34}
(2021)
\end{botherref}
\endbibitem

\bibitem[\protect\citeauthoryear{Chen et~al.}{2019}]{chen2019much}
\begin{botherref}
\oauthor{\bsnm{Chen}, \binits{Z.}},
\oauthor{\bsnm{Cao}, \binits{Y.}},
\oauthor{\bsnm{Zou}, \binits{D.}},
\oauthor{\bsnm{Gu}, \binits{Q.}}:
How much over-parameterization is sufficient to learn deep relu networks?
arXiv preprint arXiv:1911.12360
(2019)
\end{botherref}
\endbibitem

\bibitem[\protect\citeauthoryear{Lee et~al.}{2019}]{lee2019wide}
\begin{botherref}
\oauthor{\bsnm{Lee}, \binits{J.}},
\oauthor{\bsnm{Xiao}, \binits{L.}},
\oauthor{\bsnm{Schoenholz}, \binits{S.}},
\oauthor{\bsnm{Bahri}, \binits{Y.}},
\oauthor{\bsnm{Novak}, \binits{R.}},
\oauthor{\bsnm{Sohl-Dickstein}, \binits{J.}},
\oauthor{\bsnm{Pennington}, \binits{J.}}:
Wide neural networks of any depth evolve as linear models under gradient descent.
Advances in neural information processing systems
\textbf{32}
(2019)
\end{botherref}
\endbibitem

\bibitem[\protect\citeauthoryear{Yang}{2019}]{yang2019scaling}
\begin{botherref}
\oauthor{\bsnm{Yang}, \binits{G.}}:
Scaling limits of wide neural networks with weight sharing: Gaussian process behavior, gradient independence, and neural tangent kernel derivation.
arXiv preprint arXiv:1902.04760
(2019)
\end{botherref}
\endbibitem

\bibitem[\protect\citeauthoryear{Arora et~al.}{2019a}]{arora2019exact}
\begin{botherref}
\oauthor{\bsnm{Arora}, \binits{S.}},
\oauthor{\bsnm{Du}, \binits{S.S.}},
\oauthor{\bsnm{Hu}, \binits{W.}},
\oauthor{\bsnm{Li}, \binits{Z.}},
\oauthor{\bsnm{Salakhutdinov}, \binits{R.}},
\oauthor{\bsnm{Wang}, \binits{R.}}:
On exact computation with an infinitely wide neural net.
arXiv preprint arXiv:1904.11955
(2019)
\end{botherref}
\endbibitem

\bibitem[\protect\citeauthoryear{Arora et~al.}{2019b}]{arora2019fine}
\begin{bchapter}
\bauthor{\bsnm{Arora}, \binits{S.}},
\bauthor{\bsnm{Du}, \binits{S.}},
\bauthor{\bsnm{Hu}, \binits{W.}},
\bauthor{\bsnm{Li}, \binits{Z.}},
\bauthor{\bsnm{Wang}, \binits{R.}}:
\bctitle{Fine-grained analysis of optimization and generalization for overparameterized two-layer neural networks}.
In: \bbtitle{International Conference on Machine Learning},
pp. \bfpage{322}--\blpage{332}
(\byear{2019}).
\bcomment{PMLR}
\end{bchapter}
\endbibitem

\bibitem[\protect\citeauthoryear{Cao and Gu}{2019}]{cao2019generalization}
\begin{barticle}
\bauthor{\bsnm{Cao}, \binits{Y.}},
\bauthor{\bsnm{Gu}, \binits{Q.}}:
\batitle{Generalization bounds of stochastic gradient descent for wide and deep neural networks}.
\bjtitle{Advances in Neural Information Processing Systems}
\bvolume{32},
\bfpage{10836}--\blpage{10846}
(\byear{2019})
\end{barticle}
\endbibitem

\bibitem[\protect\citeauthoryear{Du et~al.}{2019}]{du2019graph}
\begin{barticle}
\bauthor{\bsnm{Du}, \binits{S.S.}},
\bauthor{\bsnm{Hou}, \binits{K.}},
\bauthor{\bsnm{Salakhutdinov}, \binits{R.R.}},
\bauthor{\bsnm{Poczos}, \binits{B.}},
\bauthor{\bsnm{Wang}, \binits{R.}},
\bauthor{\bsnm{Xu}, \binits{K.}}:
\batitle{Graph neural tangent kernel: Fusing graph neural networks with graph kernels}.
\bjtitle{Advances in Neural Information Processing Systems}
\bvolume{32},
\bfpage{5723}--\blpage{5733}
(\byear{2019})
\end{barticle}
\endbibitem

\bibitem[\protect\citeauthoryear{Wang et~al.}{2022}]{wang2022deep}
\begin{barticle}
\bauthor{\bsnm{Wang}, \binits{H.}},
\bauthor{\bsnm{Huang}, \binits{W.}},
\bauthor{\bsnm{Wu}, \binits{Z.}},
\bauthor{\bsnm{Tong}, \binits{H.}},
\bauthor{\bsnm{Margenot}, \binits{A.J.}},
\bauthor{\bsnm{He}, \binits{J.}}:
\batitle{Deep active learning by leveraging training dynamics}.
\bjtitle{Advances in Neural Information Processing Systems}
\bvolume{35},
\bfpage{25171}--\blpage{25184}
(\byear{2022})
\end{barticle}
\endbibitem

\bibitem[\protect\citeauthoryear{Hron et~al.}{2020}]{hron2020infinite}
\begin{bchapter}
\bauthor{\bsnm{Hron}, \binits{J.}},
\bauthor{\bsnm{Bahri}, \binits{Y.}},
\bauthor{\bsnm{Sohl-Dickstein}, \binits{J.}},
\bauthor{\bsnm{Novak}, \binits{R.}}:
\bctitle{Infinite attention: Nngp and ntk for deep attention networks}.
In: \bbtitle{International Conference on Machine Learning},
pp. \bfpage{4376}--\blpage{4386}
(\byear{2020}).
\bcomment{PMLR}
\end{bchapter}
\endbibitem

\bibitem[\protect\citeauthoryear{Chen et~al.}{2022}]{chen2022deep}
\begin{barticle}
\bauthor{\bsnm{Chen}, \binits{W.}},
\bauthor{\bsnm{Huang}, \binits{W.}},
\bauthor{\bsnm{Gong}, \binits{X.}},
\bauthor{\bsnm{Hanin}, \binits{B.}},
\bauthor{\bsnm{Wang}, \binits{Z.}}:
\batitle{Deep architecture connectivity matters for its convergence: A fine-grained analysis}.
\bjtitle{Advances in neural information processing systems}
\bvolume{35},
\bfpage{35298}--\blpage{35312}
(\byear{2022})
\end{barticle}
\endbibitem

\bibitem[\protect\citeauthoryear{Franceschi et~al.}{2022}]{franceschi2022neural}
\begin{bchapter}
\bauthor{\bsnm{Franceschi}, \binits{J.-Y.}},
\bauthor{\bsnm{De~B{\'e}zenac}, \binits{E.}},
\bauthor{\bsnm{Ayed}, \binits{I.}},
\bauthor{\bsnm{Chen}, \binits{M.}},
\bauthor{\bsnm{Lamprier}, \binits{S.}},
\bauthor{\bsnm{Gallinari}, \binits{P.}}:
\bctitle{A neural tangent kernel perspective of gans}.
In: \bbtitle{International Conference on Machine Learning},
pp. \bfpage{6660}--\blpage{6704}
(\byear{2022}).
\bcomment{PMLR}
\end{bchapter}
\endbibitem

\bibitem[\protect\citeauthoryear{Nguyen et~al.}{2021}]{nguyen2021benefits}
\begin{barticle}
\bauthor{\bsnm{Nguyen}, \binits{T.V.}},
\bauthor{\bsnm{Wong}, \binits{R.K.}},
\bauthor{\bsnm{Hegde}, \binits{C.}}:
\batitle{Benefits of jointly training autoencoders: An improved neural tangent kernel analysis}.
\bjtitle{IEEE Transactions on Information Theory}
\bvolume{67}(\bissue{7}),
\bfpage{4669}--\blpage{4692}
(\byear{2021})
\end{barticle}
\endbibitem

\bibitem[\protect\citeauthoryear{Ziyin et~al.}{2022}]{ziyin2022stochastic}
\begin{botherref}
\oauthor{\bsnm{Ziyin}, \binits{L.}},
\oauthor{\bsnm{Zhang}, \binits{H.}},
\oauthor{\bsnm{Meng}, \binits{X.}},
\oauthor{\bsnm{Lu}, \binits{Y.}},
\oauthor{\bsnm{Xing}, \binits{E.}},
\oauthor{\bsnm{Ueda}, \binits{M.}}:
Stochastic neural networks with infinite width are deterministic.
arXiv preprint arXiv:2201.12724
(2022)
\end{botherref}
\endbibitem

\bibitem[\protect\citeauthoryear{Huang et~al.}{2023}]{huang2023analyzing}
\begin{botherref}
\oauthor{\bsnm{Huang}, \binits{W.}},
\oauthor{\bsnm{Liu}, \binits{C.}},
\oauthor{\bsnm{Chen}, \binits{Y.}},
\oauthor{\bsnm{Da~Xu}, \binits{R.Y.}},
\oauthor{\bsnm{Zhang}, \binits{M.}},
\oauthor{\bsnm{Weng}, \binits{T.-W.}}:
Analyzing deep pac-bayesian learning with neural tangent kernel: Convergence, analytic generalization bound, and efficient hyperparameter selection.
Transactions on Machine Learning Research
(2023)
\end{botherref}
\endbibitem

\bibitem[\protect\citeauthoryear{Clerico et~al.}{2023}]{clerico2023wide}
\begin{bchapter}
\bauthor{\bsnm{Clerico}, \binits{E.}},
\bauthor{\bsnm{Deligiannidis}, \binits{G.}},
\bauthor{\bsnm{Doucet}, \binits{A.}}:
\bctitle{Wide stochastic networks: Gaussian limit and pac-bayesian training}.
In: \bbtitle{International Conference on Algorithmic Learning Theory},
pp. \bfpage{447}--\blpage{470}
(\byear{2023}).
\bcomment{PMLR}
\end{bchapter}
\endbibitem

\bibitem[\protect\citeauthoryear{Alemi et~al.}{2018}]{alemi2018fixing}
\begin{bchapter}
\bauthor{\bsnm{Alemi}, \binits{A.}},
\bauthor{\bsnm{Poole}, \binits{B.}},
\bauthor{\bsnm{Fischer}, \binits{I.}},
\bauthor{\bsnm{Dillon}, \binits{J.}},
\bauthor{\bsnm{Saurous}, \binits{R.A.}},
\bauthor{\bsnm{Murphy}, \binits{K.}}:
\bctitle{Fixing a broken elbo}.
In: \bbtitle{International Conference on Machine Learning},
pp. \bfpage{159}--\blpage{168}
(\byear{2018}).
\bcomment{PMLR}
\end{bchapter}
\endbibitem

\bibitem[\protect\citeauthoryear{Dai and Wipf}{2019}]{dai2019diagnosing}
\begin{botherref}
\oauthor{\bsnm{Dai}, \binits{B.}},
\oauthor{\bsnm{Wipf}, \binits{D.}}:
Diagnosing and enhancing vae models.
arXiv preprint arXiv:1903.05789
(2019)
\end{botherref}
\endbibitem

\bibitem[\protect\citeauthoryear{Rolinek et~al.}{2019}]{rolinek2019variational}
\begin{bchapter}
\bauthor{\bsnm{Rolinek}, \binits{M.}},
\bauthor{\bsnm{Zietlow}, \binits{D.}},
\bauthor{\bsnm{Martius}, \binits{G.}}:
\bctitle{Variational autoencoders pursue pca directions (by accident)}.
In: \bbtitle{Proceedings of the IEEE/CVF Conference on Computer Vision and Pattern Recognition},
pp. \bfpage{12406}--\blpage{12415}
(\byear{2019})
\end{bchapter}
\endbibitem

\bibitem[\protect\citeauthoryear{Kumar and Poole}{2020}]{kumar2020implicit}
\begin{bchapter}
\bauthor{\bsnm{Kumar}, \binits{A.}},
\bauthor{\bsnm{Poole}, \binits{B.}}:
\bctitle{On implicit regularization in beta-vae}.
In: \bbtitle{International Conference on Machine Learning},
pp. \bfpage{5480}--\blpage{5490}
(\byear{2020}).
\bcomment{PMLR}
\end{bchapter}
\endbibitem

\bibitem[\protect\citeauthoryear{Nakagawa et~al.}{2021}]{nakagawa2021quantitative}
\begin{bchapter}
\bauthor{\bsnm{Nakagawa}, \binits{A.}},
\bauthor{\bsnm{Kato}, \binits{K.}},
\bauthor{\bsnm{Suzuki}, \binits{T.}}:
\bctitle{Quantitative understanding of vae as a non-linearly scaled isometric embedding}.
In: \bbtitle{International Conference on Machine Learning},
pp. \bfpage{7916}--\blpage{7926}
(\byear{2021}).
\bcomment{PMLR}
\end{bchapter}
\endbibitem

\bibitem[\protect\citeauthoryear{Wipf}{2023}]{wipf2023marginalization}
\begin{botherref}
\oauthor{\bsnm{Wipf}, \binits{D.}}:
Marginalization is not marginal: No bad vae local minima when learning optimal sparse representations
(2023)
\end{botherref}
\endbibitem

\bibitem[\protect\citeauthoryear{Dai et~al.}{2021}]{dai2021value}
\begin{barticle}
\bauthor{\bsnm{Dai}, \binits{B.}},
\bauthor{\bsnm{Wenliang}, \binits{L.}},
\bauthor{\bsnm{Wipf}, \binits{D.}}:
\batitle{On the value of infinite gradients in variational autoencoder models}.
\bjtitle{Advances in Neural Information Processing Systems}
\bvolume{34},
\bfpage{7180}--\blpage{7192}
(\byear{2021})
\end{barticle}
\endbibitem

\bibitem[\protect\citeauthoryear{Dai et~al.}{2020}]{dai2020usual}
\begin{bchapter}
\bauthor{\bsnm{Dai}, \binits{B.}},
\bauthor{\bsnm{Wang}, \binits{Z.}},
\bauthor{\bsnm{Wipf}, \binits{D.}}:
\bctitle{The usual suspects? reassessing blame for vae posterior collapse}.
In: \bbtitle{International Conference on Machine Learning},
pp. \bfpage{2313}--\blpage{2322}
(\byear{2020}).
\bcomment{PMLR}
\end{bchapter}
\endbibitem

\bibitem[\protect\citeauthoryear{Higgins et~al.}{2016}]{higgins2016beta}
\begin{botherref}
\oauthor{\bsnm{Higgins}, \binits{I.}},
\oauthor{\bsnm{Matthey}, \binits{L.}},
\oauthor{\bsnm{Pal}, \binits{A.}},
\oauthor{\bsnm{Burgess}, \binits{C.}},
\oauthor{\bsnm{Glorot}, \binits{X.}},
\oauthor{\bsnm{Botvinick}, \binits{M.}},
\oauthor{\bsnm{Mohamed}, \binits{S.}},
\oauthor{\bsnm{Lerchner}, \binits{A.}}:
beta-vae: Learning basic visual concepts with a constrained variational framework
(2016)
\end{botherref}
\endbibitem

\bibitem[\protect\citeauthoryear{Reed et~al.}{2015}]{reed2015deep}
\begin{botherref}
\oauthor{\bsnm{Reed}, \binits{S.E.}},
\oauthor{\bsnm{Zhang}, \binits{Y.}},
\oauthor{\bsnm{Zhang}, \binits{Y.}},
\oauthor{\bsnm{Lee}, \binits{H.}}:
Deep visual analogy-making.
Advances in neural information processing systems
\textbf{28}
(2015)
\end{botherref}
\endbibitem

\bibitem[\protect\citeauthoryear{Chen et~al.}{}]{chen2615isolating}
\begin{botherref}
\oauthor{\bsnm{Chen}, \binits{R.T.}},
\oauthor{\bsnm{Li}, \binits{X.}},
\oauthor{\bsnm{Grosse}, \binits{R.}},
\oauthor{\bsnm{Duvenaud}, \binits{D.}}:
Isolating sources of disentanglement in vaes.
In: Proceedings of the 32nd International Conference on Neural Information Processing Systems,
vol. 2615,
p. 2625
\end{botherref}
\endbibitem

\bibitem[\protect\citeauthoryear{Kumar et~al.}{2017}]{kumar2017variational}
\begin{botherref}
\oauthor{\bsnm{Kumar}, \binits{A.}},
\oauthor{\bsnm{Sattigeri}, \binits{P.}},
\oauthor{\bsnm{Balakrishnan}, \binits{A.}}:
Variational inference of disentangled latent concepts from unlabeled observations.
arXiv preprint arXiv:1711.00848
(2017)
\end{botherref}
\endbibitem

\bibitem[\protect\citeauthoryear{Kim and Mnih}{2018}]{kim2018disentangling}
\begin{bchapter}
\bauthor{\bsnm{Kim}, \binits{H.}},
\bauthor{\bsnm{Mnih}, \binits{A.}}:
\bctitle{Disentangling by factorising}.
In: \bbtitle{International Conference on Machine Learning},
pp. \bfpage{2649}--\blpage{2658}
(\byear{2018}).
\bcomment{PMLR}
\end{bchapter}
\endbibitem

\bibitem[\protect\citeauthoryear{Liu et~al.}{2015}]{liu2015deep}
\begin{bchapter}
\bauthor{\bsnm{Liu}, \binits{Z.}},
\bauthor{\bsnm{Luo}, \binits{P.}},
\bauthor{\bsnm{Wang}, \binits{X.}},
\bauthor{\bsnm{Tang}, \binits{X.}}:
\bctitle{Deep learning face attributes in the wild}.
In: \bbtitle{Proceedings of the IEEE International Conference on Computer Vision},
pp. \bfpage{3730}--\blpage{3738}
(\byear{2015})
\end{bchapter}
\endbibitem

\end{thebibliography}
\end{document}